\documentclass[a4paper,11pt]{article}

\usepackage[margin=1in,footskip=0.5in]{geometry}
\usepackage[utf8]{inputenc} 
\usepackage[T1]{fontenc}    
\usepackage{hyperref}       
\usepackage{url}            
\usepackage{booktabs}       
\usepackage{amsfonts}       
\usepackage{nicefrac}       
\usepackage{microtype}      
\usepackage{xcolor}         
\usepackage{import}

\makeatletter
\def\thanks#1{\protected@xdef\@thanks{\@thanks
        \protect\footnotetext{#1}}}
\makeatother

\title{Approximation Bounds for Transformer Networks with Application to Regression}

\author{
Yuling Jiao, Yanming Lai, Defeng Sun, Yang Wang, Bokai Yan
\thanks{Yuling Jiao is with the School of Artificial Intelligence and the School of Mathematics and Statistics, Wuhan University, Wuhan, China (email: yulingjiaomath@whu.edu.cn).}
\thanks{Yanming Lai, Yang Wang and Bokai Yan are with the Department of Mathematics, Hong Kong University of Science and Technology, Clear Water Bay, Hong Kong, China (email: ylaiam@connect.ust.hk, yangwang@ust.hk, byanac@connect.ust.hk).}
\thanks{Defeng Sun is with the Department of Applied Mathematics and the Research Center for Intelligent Operations Research, The Hong Kong Polytechnic University, Hung Hom, Hong Kong, China (email: defeng.sun@polyu.edu.hk).}
}

\date{\vspace{-5ex}}

\begin{document}
\maketitle

\begin{abstract}
We explore the approximation capabilities of Transformer networks for H\"older and Sobolev functions, and apply these results to address nonparametric regression estimation with dependent observations. First, we establish novel upper bounds for standard Transformer networks approximating sequence-to-sequence mappings whose component functions are H\"older continuous with smoothness index $\gamma \in (0,1]$. To achieve an approximation error $\varepsilon$ under the $L^p$-norm for $p \in [1, \infty]$, it suffices to use a fixed-depth Transformer network whose total number of parameters scales as $\varepsilon^{-d_x n / \gamma}$. This result not only extends existing findings to include the case $p = \infty$, but also matches the best known upper bounds on number of parameters previously obtained for fixed-depth FNNs and RNNs. Similar bounds are also derived for Sobolev functions. Second, we derive explicit convergence rates for the nonparametric regression problem under various $\beta$-mixing data assumptions, which allow the dependence between observations to weaken over time. Our bounds on the sample complexity impose no constraints on weight magnitudes. Lastly, we propose a novel proof strategy to establish approximation bounds, inspired by the Kolmogorov-Arnold representation theorem. We show that if the self-attention layer in a Transformer can perform column averaging, the network can approximate sequence-to-sequence H\"older functions, offering new insights into the interpretability of self-attention mechanisms.

\end{abstract}

\section{Introduction}

Transformers \cite{vaswani2017attention} have become the cornerstone of modern deep learning, driving breakthroughs across multiple domains, including natural language processing \cite{devlin2019bert}, large language models \cite{openai2023gpt4}, computer vision \cite{dosovitskiy2021image}, and generative models \cite{peebles2023dit}. Although their high performance has led to widespread use in practice, significant theoretical efforts are still underway to explain exactly what contributes to their success.

An important aspect of Transformers is their expressive capacity, which refers to their ability to effectively approximate target functions. As early as the 1980s, researchers established the universal approximation property for neural networks (e.g., \cite{cybenko1989approximation, hornik1989multilayer}), demonstrating that feed-forward neural networks (FNNs) can approximate any continuous function to any precision. With the rise of deep neural networks in recent years, many works have focused on the approximation theory of neural networks. For example, \cite{yarotsky2017error, lu2021deep, jiao2023deep} studied approximation rates of deep ReLU FNNs for smooth functions, while \cite{yang2024nonparametric, yang2024optimal} and \cite{jiao2024approximation} examined, respectively, shallow ReLU FNNs and ReLU recurrent neural networks (RNNs). However, analyses based on Transformer architectures have rarely been observed. For a representative example, \cite{yun2019transformers} showed the universal approximation of Transformers, which can approximate sequence-to-sequence continuous functions under the $L^p$-norm with $p \in [1,\infty)$. \cite{kajitsuka2023transformers} revealed that even a Transformer with a single attention layer is a universal approximator. \cite{takakura2023approximation} investigated a special class of Transformers with infinite dimensional inputs. \cite{xu2024expressive} established the approximation rates of looped Transformers, which reuse the same Transformer layer iteratively, by defining the modulus of continuity for sequence-to-sequence functions. \cite{jiang2024approximation} derived approximation rate estimates for continuous Transformers by defining novel complexity measures for nonlinear sequence relationships. Most recently, \cite{takeshita2025approximation} showed that Transformers can approximate column-symmetric polynomials. Another line of research has explored replacing the softmax function in the attention mechanism with more tractable alternatives \cite{gurevych2022rate, jiao2024convergence, havrilla2024understanding}. Despite these advances, the approximation rates for general functions and the performance under the $L^\infty$-norm using the standard Transformer architecture still remain unclear.

Another significant concern pertains to Transformer performance in sequence modeling, specifically regarding how Transformers capture relationships within sequential data. Recent theoretical studies have provided diverse insights into this issue. For example, \cite{edelman2022inductive} analyzed the inductive biases inherent in self-attention mechanisms, demonstrating that sample complexity scales only logarithmically in the sequence length. \cite{chiang2022overcoming} introduced architectural adjustments to Transformers, enabling exact recognition of formal languages through enhanced long-range dependency modeling. \cite{bietti2023birth} offered an interpretation of Transformer weight matrices as associative memory systems, distinguishing between long-term parametric storage and short-term contextual memory. \cite{petrov2024prompting} proved that a pretrained Transformer, when appropriately prompted or prefix-tuned, can approximate any sequence-to-sequence function. \cite{hu2025fundamental} showed that soft prompt tuning yields a universal approximator for Lipschitz sequence-to-sequence mappings. Furthermore, \cite{wang2024understanding} conducted a rigorous analysis emphasizing how architectural components such as depth, attention heads, and feed-forward layers influence performance in tasks necessitating extensive, sparse memories. Recent developments in nonparametric regression estimation based on neural networks have mainly relied on assumptions of independent and identically distributed (i.i.d.) observations drawn from an unknown distribution \cite{suzuki2018adaptivity, schmidt2020nonparametric, nakada2020adaptive, kohler2021rate, farrell2021deep, jiao2023deep, yang2024nonparametric}. However, sequential tasks typically exhibit temporal dependence, making the i.i.d. assumption too restrictive. Some recent studies have relaxed this assumption by considering that observations are drawn from a stationary mixing distribution, where the dependence between observations weakens over time \cite{feng2023over, ren2024statistical, jiao2024approximation, jiao2025deep}. Despite these advancements, a gap remains regarding the capability of Transformers, explicitly designed to handle sequential data and temporal dependencies, in regression tasks involving dependent observations.

In this paper, we investigate the approximation of H\"older and Sobolev functions using Transformer networks and study nonparametric regression estimation under dependent observations. Our main contributions are as follows:

\begin{itemize}[itemsep=0em, labelwidth=1em, leftmargin=!]
\item We derive novel upper bounds on the approximation of standard Transformer architectures for H\"older and Sobolev functions. Specifically, to approximate a sequence-to-sequence mapping, where each component function is H\"older continuous with smoothness index $\gamma \in (0,1]$, to approximation error $\varepsilon$ under the $L^p$-norm for $p \in [1, \infty]$, it suffices to use a Transformer network whose total number of parameters scales as $\varepsilon^{-d_x n / \gamma}$. Our result establishes explicit approximation rates and extends existing findings to include the case $p = \infty$. Moreover, the number of parameters matches the best known upper bounds previously established for fixed-depth FNNs and RNNs. Similar results are also derived for Sobolev functions.

\item We present a comprehensive error analysis for the nonparametric regression problem with weakly dependent data. We achieve rates of $m^{-\frac{\gamma}{\gamma + d_x n}}$, $m^{-\frac{r \gamma}{(r+2) \gamma + (r+1) d_x n}}$ and $m^{-\frac{\gamma}{\gamma + d_x n}}$ up to logarithmic factors corresponding respectively to geometrically $\beta$-mixing, algebraically $\beta$-mixing, and i.i.d. data assumptions, where $m$ denotes the sample size and the parameter $r$ controls the strength of dependence. We also establish upper bounds on the sample complexity of Transformer networks, notably without imposing constraints on the weight size of the network.

\item We propose a novel proof strategy for establishing approximation bounds of Transformer networks, inspired by the Kolmogorov-Arnold representation theorem. By observing that the self-attention layers merely compute column-wise averages in our analysis, we demonstrate that the softmax function can be generalized to broader alternatives. This viewpoint provides new insights into the interpretability of self-attention mechanisms.

\end{itemize}

\subsection{Organization}

The rest of the paper is organized as follows. In Section \ref{sec: 4}, we define the Transformer architecture, describe the setup of the nonparametric regression problem, and list our main results. In Section \ref{sec: 5}, we present discussions and related works. All proofs are provided in Section \ref{sec: 6}.

\section{Summary of Results}\label{sec: 4}

\textit{Notation}. We use bold lowercase letters to represent vectors and bold uppercase letters to represent matrices. For any vector $\bm{v} \in \mathbb{R}^d$, we denote by $v_i$ the $i$-the element of $\bm{v}$. For any matrix $\bm{A} \in \mathbb{R}^{d \times n}$, we denote its $i$-th row by $\bm{A}_{i, :}$, its $j$-th column by $\bm{A}_{:, j}$ and the element at its $i$-th row and $j$-th column by $A_{i, j}$. We denote the all-zero and all-one vectors of length $n$ by $\bm{0}_n$ and $\bm{1}_n$, respectively. The identity matrix of size $n$ is denoted by $\bm{I}_n$. The zero matrix of size $m \times n$ is denoted by $\bm{O}_{m,n}$. When the dimensions are clear from the context, we omit the subscripts for brevity. For $m \in \mathbb{N}$, we write $[m] := {1, \dots, m}$. We use $\mathbb{N}_0$ to denote the set of nonnegative integers and $\mathbb{N}_0^d = \{(\alpha_1, \alpha_2, \ldots, \alpha_d): \alpha_k \in \mathbb{N}_0, \forall k \in [d]\}$ to denote the set of $d$-dimensional multi-index. For a multi-index $\bm{\alpha} \in \mathbb{N}_0^d$, we denote $\|\bm{\alpha}\|_{\ell^1} = \alpha_1+\alpha_2+\cdots+\alpha_d$. For a finite set $\mathbb{G}$, we use $|\mathbb{G}|$ to denote its cardinality. For two sequences $\{a_n\}$ and $\{b_n\}$, we use the notation $a_n \lesssim b_n$ and $a_n \gtrsim b_n$ to indicate $a_n \leq c_1 b_n$ and $a_n \geq c_2 b_n$, respectively, for some constants $c_1, c_2 >0$ that are independent of $n$. Furthermore, $a_n \asymp b_n$ means that both $a_n \lesssim b_n$ and $a_n \gtrsim b_n$ hold. In our analysis, we use $\sigma_S$ to represent the column-wise softmax function. Specifically, for a matrix $\bm{A} \in \mathbb{R}^{d \times n}$, $\sigma_S[\bm{A}] \in \mathbb{R}^{d \times n}$ is computed as $\sigma_S[\bm{A}]_{i, j}:=\exp (A_{i, j}) / \sum_{k=1}^d \exp (A_{k, j})$. The ReLU activation function is denoted by $\sigma_R[x] := \max \{x, 0\}$. In contrast to $\sigma_S$, $\sigma_R$ operates element-wise, regardless of whether the input is a vector or a matrix. Let $\Omega \subseteq \mathbb{R}^{d \times n}$ be a bounded domain. For $1 \leq p < \infty$, the $L^p$-norm of a real-valued function $f: \mathbb{R}^{d \times n} \to \mathbb{R}$ is defined as $\|f\|_{L^p(\Omega)} := (\int_{\Omega} |f(\bm{X})|^p \, d\bm{X})^{1/p}$, and for $p = \infty$, it is given by $\|f\|_{L^{\infty}(\Omega)} := \operatorname{ess\, sup}_{\bm{X} \in \Omega} |f(\bm{X})|$. For a matrix-valued function $\bm{F}: \mathbb{R}^{d \times n} \to \mathbb{R}^{m \times n}$, the $L^p$-norm is defined as $\|\bm{F}\|_{L^p(\Omega)} := (\int_{\Omega} \|\bm{F}(\bm{X})\|_F^p \, d\bm{X})^{1/p}$ for $1 \leq p < \infty$, and for $p = \infty$, $\|\bm{F}\|_{L^{\infty}(\Omega)} := \operatorname{ess\, sup}_{\bm{X} \in \Omega} \|\bm{F}(\bm{X})\|_F$.

\subsection{Approximation Rates for H\"older and Sobolev Functions}

We begin by introducing the architecture of Transformers, following the notations in \cite{kim2023provable} and \cite{kajitsuka2024optimal}. A Transformer network is a sequence-to-sequence function $\mathbb{R}^{d_x \times n} \to \mathbb{R}^{d_y \times n}$, comprising three main components: the self-attention layer, the (token-wise) feed-forward layer, and the embedding layer.

\textbf{Embedding and projection layer:} For embedding dimension $D \in \mathbb{N}$, the embedding and projection layers connect the input, hidden, and output spaces. The embedding layer $\mathcal{E}_{in}: \mathbb{R}^{d_x \times n} \rightarrow \mathbb{R}^{D \times n}$ is defined as
\begin{align*}
\mathcal{E}_{in}(\boldsymbol{X}) := \boldsymbol{E}_{in} \boldsymbol{X} + \bm{P} \in \mathbb{R}^{D \times n},
\end{align*}
where $\boldsymbol{E}_{in} \in \mathbb{R}^{D \times d_x}$ is a learnable weight matrix, and $\bm{P} \in \mathbb{R}^{D \times n}$ is a trainable positional encoding matrix. Since self-attention and feed-forward layers are permutation equivariant, $\bm{P}$ is introduced to provide positional information and break this equivariance. The projection layer $\mathcal{E}_{out}: \mathbb{R}^{D \times n} \rightarrow \mathbb{R}^{d_x \times n}$ is defined as 
\begin{align*}
\mathcal{E}_{out}(\boldsymbol{Y}) := \boldsymbol{E}_{out} \boldsymbol{Y} \in \mathbb{R}^{d_y \times n},
\end{align*}
where $\boldsymbol{E}_{out} \in \mathbb{R}^{d_y \times D}$ maps the high-dimensional hidden representation onto the output space.

\textbf{Self-attention layer}: Given a sequence $\boldsymbol{Z} \in \mathbb{R}^{D \times n}$, composed of $n$ tokens, each with an embedding dimension $D$, the $l$-th self-attention layer $\mathcal{F}_l^{(SA)}: \mathbb{R}^{D \times n} \to \mathbb{R}^{D \times n}$ is defined as
\begin{align*}
\mathcal{F}_l^{(SA)}(\boldsymbol{Z}) := \boldsymbol{Z} + \sum_{h=1}^H \boldsymbol{W}_{h, l}^{(O)} \left(\boldsymbol{W}_{h, l}^{(V)} \boldsymbol{Z}\right) \sigma_S \left[\left(\boldsymbol{W}_{h, l}^{(K)} \boldsymbol{Z}\right)^{\top} \left(\boldsymbol{W}_{h, l}^{(Q)} \boldsymbol{Z}\right)\right] \in \mathbb{R}^{D \times n},
\end{align*}
where $\boldsymbol{W}_{h, l}^{(V)}, \boldsymbol{W}_{h, l}^{(K)}, \boldsymbol{W}_{h, l}^{(Q)} \in \mathbb{R}^{S \times D}$ and $\boldsymbol{W}_{h, l}^{(O)} \in \mathbb{R}^{D \times S}$ are the value, key, query, and projection matrices for head $h \in [H]$ with head size $S$, respectively.

\textbf{Feed-forward layer}: The output $\boldsymbol{Z} \in \mathbb{R}^{D \times n}$ of the self-attention layer is then passed to the feed-forward layer, given by
\begin{align*}
\mathcal{F}_l^{(FF)}(\boldsymbol{Z}) := \boldsymbol{Z} + \boldsymbol{W}_l^{(2)} \sigma_R \left[\boldsymbol{W}_l^{(1)} \boldsymbol{Z} + \boldsymbol{b}_l^{(1)} \bm{1}_n^\top \right] + \boldsymbol{b}_l^{(2)} \bm{1}_n^\top \in \mathbb{R}^{D \times n},
\end{align*}
where $\boldsymbol{W}_l^{(1)} \in \mathbb{R}^{W \times D}$ and $\boldsymbol{W}_l^{(2)} \in \mathbb{R}^{D \times W}$ are weight matrices with hidden dimension $W$, and $\boldsymbol{b}_l^{(1)} \in \mathbb{R}^W$, $\boldsymbol{b}_l^{(2)} \in \mathbb{R}^D$ are bias terms.

The class of Transformer networks is then defined as  
\begin{align*}
\mathcal{T}_{d_x, d_y}(D, H, S, W, L) := \left\{\mathcal{E}_{out} \circ \mathcal{F}_L^{(FF)} \circ \mathcal{F}_L^{(SA)} \circ \cdots \circ \mathcal{F}_1^{(FF)} \circ \mathcal{F}_1^{(SA)} \circ \mathcal{E}_{in}\right\},
\end{align*}
where $D$ is the embedding dimension, $H$ is the number of attention heads, $S$ is the head size, $W$ is the hidden dimension in the feed-forward layer, and $L$ is the number of Transformer layers, each consisting of a self-attention and a feed-forward sublayer. When the dimensions are clear from the context, we use the simplified notation $\mathcal{T}(D, H, S, W, L)$ for convenience. We list some basic properties of the Transformer class in Proposition \ref{pro: 1}. Let
\begin{align}\label{eq: 15}
N = D d_x + D n + d_y D + L \left( 4HSD + 2WD + W + D \right) \lesssim (HS + W) DL
\end{align}
be the total number of training parameters in the Transformer network.

The purpose of this paper is to study the approximation of H\"older and Sobolev functions by Transformer networks. We recall the definitions of H\"older and Sobolev functions with bounded norm as follows.

\begin{definition}[H\"older functions]
Let $\Omega$ be a bounded domain in $\mathbb{R}^{d_x \times n}$ and $\gamma \in (0, 1]$. Given $K_\mathcal{H} > 0$, we denote the H\"older class $\mathcal{H}^\gamma(\Omega, K_\mathcal{H})$ as
\begin{align*}
\mathcal{H}^\gamma(\Omega, K_\mathcal{H}) = \left\{f: \Omega \to \mathbb{R}: \|f\|_{L^{\infty}(\Omega)} + \sup\limits_{\bm{X}, \bm{Y} \in \Omega, \bm{X} \neq \bm{Y}} \dfrac{|f(\bm{X}) - f(\bm{Y})|}{\| \bm{X} - \bm{Y} \|_F^\gamma} \leq K_\mathcal{H} \right\}.
\end{align*}
\end{definition}

\begin{definition}[Sobolev functions]
Let $\Omega$ be a bounded domain in $\mathbb{R}^{d_x \times n}$. For $p \in [1,\infty)$ and $K_\mathcal{W} > 0$, we denote the Sobolev class $\mathcal{W}^{1,p}(\Omega, K_\mathcal{W})$ as
\begin{align*}
\mathcal{W}^{1,p}(\Omega, K_\mathcal{W}) = \left\{f: \Omega \to \mathbb{R}: \left({\textstyle \sum_{\|\bm{\alpha}\|_{\ell^1} \leq 1}} \int_\Omega \left|D^{\bm{\alpha}} f\right|^p \dd \bm{X}\right)^{1/p} \leq K_\mathcal{W} \right\},
\end{align*}
and for $p = \infty$, the Sobolev class $\mathcal{W}^{1,\infty}(\Omega, K_\mathcal{W})$ is defined as
\begin{align*}
\mathcal{W}^{1,\infty}(\Omega, K_\mathcal{W}) = \left\{f: \Omega \to \mathbb{R}: {\textstyle \sum_{\|\bm{\alpha}\|_{\ell^1} \leq 1}} \operatorname{ess\, sup}_\Omega \left|D^{\bm{\alpha}} f\right| \leq K_\mathcal{W} \right\},
\end{align*}
where $\bm{\alpha} \in \mathbb{N}_0^{d_x \times n}$ is a multi-index and $D^{\bm{\alpha}}$ is the weak derivative of order $\bm{\alpha}$.
\end{definition}

H\"older and Sobolev functions are central objects in approximation theory due to their close connection with polynomial and spline approximations \cite{devore1998nonlinear}. A variety of embedding and interpolation results relate these spaces. For example, the Sobolev embedding $\mathcal{W}^{1,p}(\Omega, K_\mathcal{W}) \hookrightarrow \mathcal{H}^{1-\frac{d}{p}}(\Omega, K_\mathcal{H})$ holds with the H\"older exponent $\gamma = 1 - \frac{d}{p}$ and an appropriate constant $K_\mathcal{H}$ depending on $K_\mathcal{W}$ and the geometry of $\Omega$. In the limiting case where $p = \infty$, the Sobolev space $\mathcal{W}^{1,\infty}(\Omega, K_\mathcal{W})$ consists of functions with essentially bounded first-order weak derivatives, which directly implies that these functions are Lipschitz continuous. In other words, $\mathcal{W}^{1,\infty}(\Omega, K_\mathcal{W}) \hookrightarrow \mathcal{H}^1(\Omega, K_\mathcal{H})$. For applications in machine learning, it is thus important to understand how efficiently Transformer networks can approximate functions in both the H\"older and Sobolev spaces.

We now present our main results on the approximation capabilities of Transformer networks for H\"older and Sobolev functions. We defer the proofs to Section \ref{sec: 1}.

\begin{theorem}\label{thm: 4}
Given $\gamma \in (0,1]$ and $K_\mathcal{H} > 0$, assume that the target function $\bm{F}: [0,1]^{d_x \times n} \rightarrow \mathbb{R}^{d_x \times n}$ satisfies $F_{i,j} \in \mathcal{H}^\gamma ([0,1]^{d_x \times n}, K_\mathcal{H})$ for each $i \in [d_x], j \in [n]$. For any $\varepsilon \in (0,1)$ and $p \in [1,\infty]$, there exists a Transformer network 
\begin{align*}
\mathcal{N} \in \mathcal{T}_{d_x, d_x}(D = C_1, H = C_2, S = C_3, W = C_4 \cdot \lceil\varepsilon^{-\frac{d_x n}{\gamma}}\rceil, L = C_5)
\end{align*}
such that
\begin{align*}
\|\mathcal{N} - \bm{F}\|_{L^p([0,1]^{d_x \times n})} \leq 4 (d_x n)^2 K_{\mathcal{H}} \varepsilon,
\end{align*}
where
\begin{enumerate}[itemsep=0em, labelwidth=1em, leftmargin=!]
\item $C_1 = d_x$, $C_2 = 1$, $C_3 = 1$, $C_4 = 5 n$ and $C_5 = 2$ if $p \in [1,\infty)$;
\item $C_1 = 5 d_x 3^{d_x n}$, $C_2 = 3^{d_x n}$, $C_3 = 1$, $C_4 = 5 n 3^{d_x n}$ and $C_5 = 2 + 2 d_x n$ if $p = \infty$.
\end{enumerate}
\end{theorem}

\begin{theorem}\label{thm: 5}
Given $p \in [1,\infty)$ and $K_\mathcal{W} > 0$, assume that the target function $\bm{F}: [0,1]^{d_x \times n} \rightarrow \mathbb{R}^{d_x \times n}$ satisfies $F_{i,j} \in \mathcal{W}^{1,p}([0,1]^{d_x \times n}, K_\mathcal{W})$ for each $i \in [d_x], j \in [n]$. For any $\varepsilon \in (0,1)$, there exists a Transformer network 
\begin{align*}
\mathcal{N} \in \mathcal{T}_{d_x, d_x}(D = d_x, H = 1, S = 1, W = 5 n \cdot \lceil\varepsilon^{-d_x n}\rceil, L = 2)
\end{align*}
such that
\begin{align*}
\|\mathcal{N} - \bm{F}\|_{L^p([0,1]^{d_x \times n})} \leq 4C (d_x n)^2 K_{\mathcal{W}} \varepsilon,
\end{align*}
where $C$ is a constant depending only on $d_x n$.
\end{theorem}

We make several remarks regarding our results. First, ever since \cite{vaswani2017attention} proposed the Transformer architecture, there have been various theoretical analyses on its expressive capacity. A series of works established the universal approximation for sequence-to-sequence continuous functions under the $L^p$-norm for $1 \leq p < \infty$. More precisely, \cite{yun2019transformers} showed that 
\begin{align*}
\sup_{\bm{F}: F_{i,j} \in \mathcal{C}(\Omega)} \inf_{\mathcal{N} \in \mathcal{T}(D, H, S, W, L)} \left\|\mathcal{N} - \bm{F}\right\|_{L^p(\Omega)} \rightarrow 0
\end{align*}
for fixed and sufficiently large $D, H, S, W$ and as $L \rightarrow \infty$, where $\mathcal{C}(\Omega)$ denotes the space of continuous functions on $\Omega$, and subsequently \cite{kajitsuka2023transformers} showed that
\begin{align*}
\sup_{\bm{F}: F_{i,j} \in \mathcal{C}(\Omega)} \inf_{\mathcal{N} \in \mathcal{T}(D, H, S, W, L)} \left\|\mathcal{N} - \bm{F}\right\|_{L^p(\Omega)} \rightarrow 0
\end{align*}
for fixed and sufficiently large $D, H, S, L$ and as $W \rightarrow \infty$. We show that for $1 \leq p \leq \infty$,
\begin{align*}
\sup_{\bm{F}: F_{i,j} \in \mathcal{H}^\gamma (\Omega, K_\mathcal{H})} \inf_{\mathcal{N} \in \mathcal{T}(D, H, S, W, L)} \left\|\mathcal{N} - \bm{F}\right\|_{L^p(\Omega)} \lesssim K_{\mathcal{H}} W^{-\gamma / (d_x n)}
\end{align*}
and
\begin{align*}
\sup_{\bm{F}: F_{i,j} \in \mathcal{W}^{1,p}(\Omega, K_\mathcal{W})} \inf_{\mathcal{N} \in \mathcal{T}(D, H, S, W, L)} \left\|\mathcal{N} - \bm{F}\right\|_{L^p(\Omega)} \lesssim K_{\mathcal{W}} W^{-1 / (d_x n)},
\end{align*}
for fixed $D, H, S, L$ and adjustable $W$. Our results not only provide the approximation rates for general H\"older and Sobolev functions, but also extend to the case $p = \infty$, which previous works were unable to address. These improvements are largely attributed to the use of the horizontal shift technique, which was originally introduced in \cite{lu2021deep} and further developed in \cite{shen2020deep, zhang2022deep}. While their technique was developed for ReLU FNNs, we find that the ideas can be applied to the Transformer architecture. We summarize the related results in Table \ref{tab: 2}.

Second, We emphasize that our approximation results are established in a sequence-to-sequence sense; that is, every entry of the Transformer network $\mathcal{N}$ simultaneously approximates the corresponding entry of the matrix-valued target function $\bm{F}$. It is not hard to extend the target function $\bm{F}: [0,1]^{d_x \times n} \rightarrow \mathbb{R}^{d_y \times n}$ to general $d_y \in \mathbb{N}$ in Theorem \ref{thm: 4} and \ref{thm: 5}.

Our results further demonstrate that Transformer networks possess stronger expressive capabilities than RNNs in approximating sequence-to-sequence functions. As discussed in \cite{hoon2023minimal, jiao2024approximation}, RNNs are inherently limited to approximating past-dependent sequence-to-sequence functions because, at each time step, only the current and past tokens are utilized, leaving future tokens unprocessed. In contrast, Transformers have the advantage of accessing the entire input sequence. In other words, even the first output token of a Transformer depends on the entire input sequence, whereas in an RNN the first output token depends only on the first input token, the second on the first two, and so forth, owing to the sequential nature of RNNs. This distinction underpins our assertion that Transformer architectures outperform RNNs in terms of expressive power.

Third, we observe that to achieve an approximation error of $\varepsilon$, the total number of training parameters required scales as $\varepsilon^{-d_x n / \gamma}$ for H\"older functions and as $\varepsilon^{-d_x n}$ for Sobolev functions, matching the best known upper bounds previously established for fixed-depth FNNs and RNNs with input dimension $d_x n$ \cite{yarotsky2017error, shen2020deep, lu2021deep, jiao2023deep, siegel2023optimal, jiao2024approximation}.

\begin{table}[t]
    \caption{Comparison of approximation rates.}
    \label{tab: 2}
    \centering
    \begin{tabular}{p{2.3cm}<{\centering\arraybackslash} p{1.5cm}<{\centering\arraybackslash} p{1.9cm}<{\centering\arraybackslash} p{1.1cm}<{\centering\arraybackslash} p{2.8cm}<{\centering\arraybackslash} p{2cm}<{\centering\arraybackslash} p{1.5cm}<{\centering\arraybackslash}}
    \toprule[1pt]
    \textbf{Reference} & \textbf{Type} & \textbf{Target Function}$^1$ & \textbf{Metric}$^2$ & \textbf{Activations in Self-Attention Layers}$^3$ & \textbf{Width}$^4$ & \textbf{Depth} \\
    \midrule
    \cite{yun2019transformers} & Universality & $\mathcal{C}^0$ & $L^p$ & $\sigma_{S}$ with bias &  & \\
    \cmidrule{1-5}
    \cite{kajitsuka2023transformers} & Universality & $\mathcal{C}^0$ & $L^p$ & $\sigma_{S}$ &  &  \\
    \cmidrule{1-5}
    \cite{fang2022attention} & Universality & $\mathcal{C}^0$ & $L^\infty$ & $\sigma_{H}$ &  &  \\
    \midrule
    \cite{jiao2024convergence} & Rate & \makecell[c]{$\mathcal{H}^\gamma$\\ $\mathcal{C}^m$} & $L^\infty$ & $X \odot \sigma_H(X)$ & \makecell[c]{$\mathcal{O}(\varepsilon^{-d_x n/\gamma})$\\ $\mathcal{O}(\varepsilon^{-d_x n/m})$} & \makecell[c]{$\mathcal{O}(\log \frac{1}{\varepsilon})$\\ $\mathcal{O}(\log \frac{1}{\varepsilon})$} \\
    \midrule
    \cite{havrilla2024understanding} & Rate & $\mathcal{H}^\gamma$ & $L^\infty$ & $\sigma_{R}$ & $\mathcal{O}(\varepsilon^{-d_x n/\gamma})$ & $\mathcal{O}(\log \frac{1}{\varepsilon})$ \\
    \midrule[0.7pt]
    \makecell[c]{\textbf{Ours}\\ (Theorem \ref{thm: 4})} & Rate & $\mathcal{H}^\gamma$ & $L^\infty$ & $\sigma_{S}$ & $\mathcal{O}(\varepsilon^{-d_x n/\gamma})$ & $\mathcal{O}(1)$ \\
    \midrule
    \makecell[c]{\textbf{Ours}\\ (Theorem \ref{thm: 5})} & Rate & $\mathcal{W}^{1,p}$ & $L^p$ & $\sigma_{S}$ & $\mathcal{O}(\varepsilon^{-d_x n})$ & $\mathcal{O}(1)$ \\
    \bottomrule[1pt]
    \end{tabular}
\vspace{0.1em} \raggedright
\footnotesize $^1$The space $\mathcal{C}^m$ consists of all functions whose first $m$ derivatives exist and are continuous, and $\mathcal{C}^0$ denotes the space of continuous functions.
$^2$$p \in [1,\infty)$.
$^3$Different Transformer architectures are obtained by replacing the softmax function in the self-attention layer with various activation functions. The symbol $\odot$ denotes the Hadamard product.
$^4$Following \cite{kim2023provable}, the width of a Transformer network is defined as $\max\{D, HS, W\}$. We omit constants independent of $\varepsilon \in (0,1)$.
\end{table}

\subsection{Nonparametric Regression}

We then study the regression problem, which seeks to estimate an unknown target regression function from finite observations. We consider the following $n$-step prediction model
\begin{align}
\label{eq: 17}
Y = f^*(X_1, X_2, \ldots, X_n) + \varepsilon, 
\end{align}
where $Y \in \mathbb{R}$ is a response, $f^*(\bm{x}_1, \ldots, \bm{x}_n) = \mathbb{E}[Y | X_1=\bm{x}_1, \ldots, X_n=\bm{x}_n]: [0,1]^{d_x \times n} \rightarrow \mathbb{R}$ is an unknown regression function and $\varepsilon$ is a sub-Gaussian noise, independent of $X_i, i=1,\ldots,n$, with $\mathbb{E}[\varepsilon] = 0$ and 
\begin{align*}
\mathbb{E} [\exp(s\varepsilon)] \leq \exp\left(\frac{\sigma^2 s^2}{2}\right) \ \text{ for any } s \in \mathbb{R}.
\end{align*}
Our purpose is to estimate the unknown target regression function $f^*$ given observations $\mathcal{D}_m = \{(\bm{x}_1, y_1), \ldots, (\bm{x}_m, y_m)\}$ which may not be i.i.d.

As observed in real sequence modeling applications, the sequential observations often exhibit temporal dependence, rendering the usual i.i.d. assumption inapplicable. This motivates us to consider dependent data. A frequently used alternative is to assume that observations are drawn from a stationary mixing distribution, where the dependence between observations weakens over time. This scenario has become standard and has been discussed extensively in previous studies \cite{yu1994rates, meir2000nonparametric, mohri2008rademacher, steinwart2009fast, mohri2010stability, agarwal2012generalization, shalizi2013predictive, kuznetsov2017generalization, ren2024statistical}. We now introduce the relevant definitions.

\begin{definition}[Stationarity]
\label{definition: 1}
A sequence of random variables $\{\bm{x}_t\}_{t=-\infty}^{\infty}$ is said to be stationary if for any $t$ and non-negative integers $m$ and $k$, the random vectors $(\bm{x}_t, \ldots, \bm{x}_{t+m})$ and $(\bm{x}_{t+k}, \ldots, \bm{x}_{t+m+k})$ have the same distribution.
\end{definition}

\begin{definition}[$\beta$-mixing]
\label{definition: 2}
Let $\{\bm{x}_t\}_{t=-\infty}^{\infty}$ be a stationary sequence of random variables. For any $i, j \in \mathbb{Z} \cup \{-\infty,+\infty\}$, let $\sigma_i^j$ denote the $\sigma$-algebra generated by the random variables $\bm{x}_k, i \leq k \leq j$. Then, for any positive integer $k$, the $\beta$-mixing coefficient of the stochastic process $\{\bm{x}_t\}_{t=-\infty}^{\infty}$ is defined as
\begin{align*}
\beta(k) = \sup_n \mathbb{E}_{B \in \sigma_{-\infty}^n} \left[\sup_{A \in \sigma_{n+k}^{\infty}} |\mathbb{P}(A \mid B)-\mathbb{P}(A)|\right].
\end{align*}
$\{\bm{x}_t\}_{t=-\infty}^{\infty}$ is said to be $\beta$-mixing if $\beta(k) \rightarrow 0$ as $k \rightarrow \infty$. It is said to be algebraically $\beta$-mixing if there exist real numbers $\beta_0>0$ and $r>0$ such that $\beta(k) \leq \beta_0 / k^r$ for all $k$, and geometrically $\beta$-mixing if there exist real numbers $\beta_0, \beta_1>0$ and $r>0$ such that $\beta(k) \leq \beta_0 \exp \left(-\beta_1 k^r\right)$ for all $k$.
\end{definition}

In this work, we assume that the sequence of random variables $\mathcal{X} = \{\bm{x}_t\}_{t=1}^m$ is drawn from a stationary $\beta$-mixing process. By Definition \ref{definition: 1}, the time index $t$ does not affect the distribution of $\bm{x}_t$ in a stationary sequence. Moreover, any $n$ consecutive observations, $(\bm{x}_{t-n+1}, \ldots, \bm{x}_t)$, share the same joint distribution, which we denote by $\Pi$. We assume that $\Pi$ is supported on $[0,1]^{d_x \times n}$ and is absolutely continuous with respect to the Lebesgue measure, with its probability density function uniformly bounded above by a finite constant on $[0,1]^{d_x \times n}$.

Definition \ref{definition: 2} states that a sequence of random variables is mixing if the influence of past events on future events diminishes as the temporal gap increases. This definition provides a standard measure of the dependence among the random variables $\{\bm{x}_t\}$ within a stationary sequence. We note that in certain special cases, such as Markov chains, the mixing coefficients admit upper bounds that can be estimated from data \cite{hsu2015mixing}. If $\{\bm{x}_t\}_{t=1}^m$ are i.i.d. random variables, then by definition, $\beta(k) = 0$ for all $k \geq 1$.

A fundamental method for estimating $f^*$ is to minimize the mean squared error or the $L^2$ risk, i.e., to solve
\begin{align*}
\argmin{f}\, \mathcal{R}(f) := \mathbb{E}_{(X_1, \ldots, X_n) \sim \Pi, \,Y} [(f(X_1, \ldots, X_n) - Y)^2].
\end{align*}
Under the assumption that $\mathbb{E}[\varepsilon | X_1, \ldots, X_n] = 0$, the underlying regression function $f^*$ is the optimal solution, that is, the global minimizer of $\mathcal{R}(f)$. However, in practical applications the joint distribution of $((X_1, \ldots, X_n), Y)$ is typically unknown, and only a random sample $\mathcal{D}_m = \{(\bm{x}_i, y_i)\}_{i=1}^m$ with sample size $m$ is available. Given that each evaluation of the Transformer requires a sequence of length $n$, we consider a sliding window training approach. Specifically, we first group the observations into overlapping sequences, each of length $n$, to construct new sequences
\begin{align*}
\{((\bm{x}_1, \ldots, \bm{x}_n), y_n), ((\bm{x}_2, \ldots, \bm{x}_{n+1}), y_{n+1}), \ldots, ((\bm{x}_{m-n+1}, \ldots, \bm{x}_m), y_m)\},
\end{align*}
and then estimate the unknown target function $f^*$ using the empirical risk minimizer
\begin{align}
\label{eq: 20}
\hat{f}_m \in \argmin{f \in \mathcal{F}}\, \mathcal{R}_m(f) := \frac{1}{m-n+1} \sum_{t=n}^m (f(\bm{x}_{t-n+1}, \ldots, \bm{x}_t) - y_t)^2,
\end{align}
where we choose the hypothesis class 
\begin{align*}
\mathcal{F} = \mathcal{F}(D_m, H_m, S_m, W_m, L_m) = \{\langle\mathcal{N}(X), \bm{E}\rangle: \mathcal{N} \in \mathcal{T}_{d_x, d_x}(D_m, H_m, S_m, W_m, L_m)\}.
\end{align*}
Here, $\langle \cdot, \cdot \rangle$ denotes the matrix inner product, and $\bm{E} \in \mathbb{R}^{d_x \times n}$ is an arbitrary weight matrix. The performance of the estimator is evaluated by the excess risk, defined as the difference between the $L^2$ risks of $\hat{f}_m$ and $f^*$, given by
\begin{align*}
\mathcal{R}(\hat{f}_m) - \mathcal{R}(f^*) = \mathbb{E}_{(X_1, \ldots, X_n)} [(\hat{f}_m(X_1, \ldots, X_n) - f^*(X_1, \ldots, X_n))^2].
\end{align*}

To control the sample complexity, we require that the hypothesis class is uniformly bounded. We define the truncation operator $\mathcal{C}_B$ with level $B > 0$ for a real-valued function $f$ as
\begin{align*}
\mathcal{C}_B f(x) := \begin{cases}
f(x) & \text { if } |f(x)| \leq B, \\
\operatorname{sgn}(f(x)) \cdot B & \text { if } |f(x)| > B.
\end{cases}
\end{align*}
For a class of real-valued functions $\mathcal{F}$, we use the notation $\mathcal{C}_B \mathcal{F} := \{\mathcal{C}_B f: f \in \mathcal{F}\}$. Note that the truncation can be implemented by a feed-forward layer that applies the operation $\sigma_R[x] - \sigma_R[-x] - \sigma_R[x-B] + \sigma_R[-x-B]$ element-wise. Our next theorem provides a convergence rate for estimating the target function $f^*$ using the truncated empirical risk minimizer $\mathcal{C}_{B_m} \hat{f}_m$. We defer the proof to Section \ref{sec: 2}.

\begin{theorem}\label{thm: 7}
Under model (\ref{eq: 17}), assume that the regression function $f^* \in \mathcal{H}^\gamma([0,1]^{d_x \times n}, K_\mathcal{H})$ for some $\gamma \in (0,1]$ and $K_\mathcal{H} > 0$, and that the probability measure of the covariate $\Pi$ is supported on $[0,1]^{d_x \times n}$ and is absolutely continuous with respect to the Lebesgue measure, with its density function uniformly bounded by a finite constant. Let $\hat{f}_m$ be the empirical risk minimizer defined in (\ref{eq: 20}) over a random sample $\mathcal{D}_m = \{(\bm{x}_i, y_i)\}_{i=1}^m$. Then, the following excess risk bounds hold:

\begin{itemize}[itemsep=0em, labelwidth=1em, leftmargin=!]
\item If $\{\bm{x}_i\}_{i=1}^m$ is a geometrically $\beta$-mixing sequence, i.e., $\beta(k) \leq \beta_0 \exp \left(-\beta_1 k^r\right)$ for some $r,\beta_0,\beta_1>0$, then by choosing $B_m \asymp \log m$ and the hypothesis class
\begin{align*}
\mathcal{F}(D_m \lesssim 1, H_m \lesssim 1, S_m \lesssim 1, W_m \lesssim m^{\frac{d_x n}{2 \gamma + 2 d_x n}}, L_m \lesssim 1),
\end{align*}
we have
\begin{align*}
\mathbb{E}_{\mathcal{D}_m} [\mathcal{R}(\mathcal{C}_{B_m} \hat{f}_m) - \mathcal{R}(f^*)] \lesssim m^{-\frac{\gamma}{\gamma + d_x n}} (\log m)^{3 + 1/r}.
\end{align*}

\item If $\{\bm{x}_i\}_{i=1}^m$ is an algebraically $\beta$-mixing sequence, i.e., $\beta(k) \leq \beta_0 / k^r$ for some $r,\beta_0>0$, then by choosing $B_m \asymp \log m$ and the hypothesis class 
\begin{align*}
\mathcal{F}(D_m \lesssim 1, H_m \lesssim 1, S_m \lesssim 1, W_m \lesssim m^{\frac{r d_x n}{2(r+2) \gamma + 2(r+1) d_x n}}, L_m \lesssim 1),
\end{align*}
we have
\begin{align*}
\mathbb{E}_{\mathcal{D}_m} [\mathcal{R}(\mathcal{C}_{B_m} \hat{f}_m) - \mathcal{R}(f^*)] \lesssim m^{-\frac{r \gamma}{(r+2) \gamma + (r+1) d_x n}} (\log m)^3.
\end{align*}

\item If $\{\bm{x}_i\}_{i=1}^m$ is a sequence of i.i.d. random variables, then by choosing $B_m \asymp \log m$ and the hypothesis class 
\begin{align*}
\mathcal{F}(D_m \lesssim 1, H_m \lesssim 1, S_m \lesssim 1, W_m \lesssim m^{\frac{d_x n}{2 \gamma + 2 d_x n}}, L_m \lesssim 1),
\end{align*}
we have
\begin{align*}
\mathbb{E}_{\mathcal{D}_m} [\mathcal{R}(\mathcal{C}_{B_m} \hat{f}_m) - \mathcal{R}(f^*)] \lesssim m^{-\frac{\gamma}{\gamma + d_x n}} (\log m)^3.
\end{align*}
\end{itemize}

\end{theorem}

\begin{table}[t]
    \caption{Comparison of convergence rates.}
    \label{tab: 1}
    \centering
    \begin{tabular}{cccc}
    \toprule[1pt]
    \textbf{Reference} & \textbf{Hypothesis Class} & \textbf{Dependence Assumption} & \textbf{Convergence Rate}$^1$ \\
    \midrule
    \cite{feng2023over} & FNN & geometrically $\beta$-mixing & $\widetilde{\mathcal{O}}(m^{-\frac{\gamma}{2 \gamma + 2 d + 2}})$ \\ 
    \midrule
    \cite{ren2024statistical} & FNN & geometrically $\beta$-mixing & $\widetilde{\mathcal{O}}(m^{-\frac{2 \gamma}{2 \gamma + d}})$ \\
    \midrule
    \multirow{3}{*}{\cite{jiao2024approximation}} & \multirow{3}{*}{RNN} & geometrically $\beta$-mixing & $\widetilde{\mathcal{O}}(m^{-\frac{2 \gamma}{2 \gamma + d_x n}})$ \\
    & & algebraically $\beta$-mixing & $\widetilde{\mathcal{O}}(m^{-\frac{2 r \gamma}{(2r+4) \gamma + (r+1) d_x n}})$ \\
    & & i.i.d. & $\widetilde{\mathcal{O}}(m^{-\frac{2 \gamma}{2 \gamma + d_x n}})$ \\
    \midrule
    \multirow{3}{*}{\makecell[c]{\textbf{Ours}\\ (Theorem \ref{thm: 7})}} & \multirow{3}{*}{Transformer} & geometrically $\beta$-mixing & $\widetilde{\mathcal{O}}(m^{-\frac{\gamma}{\gamma + d_x n}})$ \\
    & & algebraically $\beta$-mixing & $\widetilde{\mathcal{O}}(m^{-\frac{r \gamma}{(r+2) \gamma + (r+1) d_x n}})$ \\
    & & i.i.d. & $\widetilde{\mathcal{O}}(m^{-\frac{\gamma}{\gamma + d_x n}})$ \\
    \bottomrule[1pt]
    \end{tabular} \\
\vspace{0.1em} \raggedright
\footnotesize $^1$We omit constants independent of $m$ and logarithmic factors in $m$. $d$ and $d_x n$ denote the input dimensions for vector and sequence inputs, respectively. 
\end{table}

It is well known that the optimal convergence rate in nonparametric regression with squared loss for i.i.d. data is $m^{-2\gamma/(d_x n+2\gamma)}$ \cite{stone1982optimal, donoho1998minimax}, and that the same rate remains optimal for certain $\beta$-mixing sequences \cite{yu1993density, viennet1997inequalities}. Therefore, the rates in Theorem \ref{thm: 7} are suboptimal. We attribute this suboptimality to the loose upper bound on the VC-dimension (see Lemma \ref{lemma: 11}). We consider the i.i.d. case for an illustration. Classical empirical process techniques yield a decomposition of the excess risk into an approximation error and a statistical error, and by trading off these two errors one obtains the optimal convergence rate, as in \cite{jiao2023deep}. For the approximation error, to approximate a H\"older function with smoothness index $\gamma$ up to accuracy $\varepsilon$, it suffices to use a ReLU FNN with a total number of adjustable parameters $N \lesssim \varepsilon^{-d/\gamma}$ (up to logarithmic factors), where $d$ denotes the input dimension and in our setting $d = d_x n$. Meanwhile, the VC-dimension, which governs the statistical error, grows linearly with $N$ (assuming fixed depth) due to the piecewise linear nature of ReLU FNNs. In fact, for FNNs with piecewise polynomial activations (such as ReLU$^k$, see \cite{bartlett2019nearly, ding2025semi}), or for self-attention layers with piecewise polynomial activation functions (e.g., by replacing the softmax $\sigma_S[\bm{Z}]$ with the hardmax $\sigma_H[\bm{Z}]$ \cite{kajitsuka2024optimal} or $\bm{Z} \odot \sigma_H[\bm{Z}]$ \cite{gurevych2022rate, jiao2024convergence}), the VC-dimension grows linearly in the total number of parameters $N$. However, for function classes involving exponential operations, such as those defined by sigmoid networks or radial basis function networks, the best known upper bounds on the VC-dimension grow quadratically in $N$ \cite{karpinski1997polynomial, anthony2009neural}. We use the same method to establish an upper bound on the VC-dimension of Transformer networks, and hence it also exhibits quadratic growth in $N$. As noted in \cite{bartlett2003vapnik}, there is a gap between the best known upper and lower bounds for function classes that involve exponential operations, and it remains open whether these bounds are optimal. \cite{karpinski1997polynomial} conjectured that the upper bounds could be improved. To prove Theorem \ref{thm: 7}, we decompose the excess risk into the approximation error and the statistical error. By Theorem \ref{thm: 4} and \eqref{eq: 15}, to achieve an approximation error of at most $\varepsilon$, it suffices to use a Transformer network with total parameters $N \lesssim \varepsilon^{-d_x n/\gamma}$. However, since the VC-dimension scales as $N^2$, it grows faster than in the ReLU FNN case, leading to the suboptimal rate after trade-off. We leave possible improvements to this gap as an open problem for future study.

We observe that, ignoring logarithmic factors, the convergence rates for the geometrically $\beta$-mixing and i.i.d. cases are identical. In addition, for the algebraically $\beta$-mixing case the convergence rate is given by $m^{-\frac{r \gamma}{(r+2) \gamma + (r+1) d_x n}}$, which improves as the mixing parameter $r$ increases. When $r$ is sufficiently large, this rate approaches $m^{-\frac{\gamma}{\gamma + d_x n}}$, matching that of the geometrically $\beta$-mixing and i.i.d. cases. This observation is consistent with the findings in \cite{jiao2024approximation}. We remark that a completely analogous theorem holds for estimating a Sobolev target function $f^* \in \mathcal{W}^{1,p}([0,1]^{d_x \times n}, K_\mathcal{W})$ for some $p \geq 2$ and $K_\mathcal{W} > 0$. We summarize the related results in Table \ref{tab: 1}.

\subsection{Approximation by Generalized Transformer Networks}

We begin by introducing generalized Transformer networks, which extend the original definitions of Transformer layers to enable more flexible functional representations.

\textbf{Generalized feed-forward layer}: We define the generalized feed-forward layer as
\begin{align*}
\mathcal{F}_l^{(GFF)}(\boldsymbol{Z}) := \boldsymbol{Z} + \boldsymbol{W}_l^{(2)} \sigma_R \left[\boldsymbol{W}_l^{(1)} \boldsymbol{Z} + \boldsymbol{B}_l^{(1)} \right] + \boldsymbol{B}_l^{(2)} \in \mathbb{R}^{D \times n},
\end{align*}
where $\boldsymbol{W}_l^{(1)} \in \mathbb{R}^{W \times D}$ and $\boldsymbol{W}_l^{(2)} \in \mathbb{R}^{D \times W}$ are weight matrices, and $\boldsymbol{B}_l^{(1)} \in \mathbb{R}^{W \times n}$, $\boldsymbol{B}_l^{(2)} \in \mathbb{R}^{D \times n}$ are bias matrices. In contrast to the standard feed-forward layer, we allow different bias terms for each column, thereby generalizing the original formulation.

\textbf{Generalized self-attention layer}: We define the generalized self-attention layer as
\begin{align*}
\mathcal{F}_l^{(GSA)}(\boldsymbol{Z}) := \boldsymbol{Z} + \sum_{h=1}^H \boldsymbol{W}_{h, l}^{(O)} \sigma_G [\boldsymbol{Z}] \in \mathbb{R}^{D \times n},
\end{align*}
where $\boldsymbol{W}_{h, l}^{(O)} \in \mathbb{R}^{D \times D}$ is a weight matrix with rank at most $S$ for all $h$ and $l$, and $\sigma_G[\boldsymbol{Z}]$ is a general function (may vary across different $h$ and $l$) with the only requirement that, for a particular parameter choice, it computes the column average of $\boldsymbol{Z}$, namely, 
\begin{align*}
\sigma_G [\boldsymbol{Z}] = \left(\frac{1}{n} \sum_{j=1}^n \boldsymbol{Z}_{:,j}, \ldots, \frac{1}{n} \sum_{j=1}^n \boldsymbol{Z}_{:,j}\right).
\end{align*}
Clearly, both softmax-based self-attention \cite{vaswani2017attention}
\begin{align*}
\sigma_G [\boldsymbol{Z}] = \boldsymbol{Z} \cdot \sigma_S \left[\left(\boldsymbol{W}^{(K)} \boldsymbol{Z}\right)^{\top} \left(\boldsymbol{W}^{(Q)} \boldsymbol{Z}\right)\right]
\end{align*}
and (averaging) hardmax-based self-attention \cite{perez2021attention}
\begin{align*}
\sigma_G [\boldsymbol{Z}] = \boldsymbol{Z} \cdot \sigma_H \left[\left(\boldsymbol{W}^{(K)} \boldsymbol{Z}\right)^{\top} \left(\boldsymbol{W}^{(Q)} \boldsymbol{Z}\right)\right]
\end{align*}
satisfy the above definition, since they compute the column average of $\boldsymbol{Z}$ when $\boldsymbol{W}^{(K)} = \boldsymbol{W}^{(Q)} = \bm{O}$.

The class of generalized Transformer networks is then defined as  
\begin{align*}
\mathcal{GT}_{d_x, d_y}(D, H, S, W, L) := \left\{\mathcal{E}_{out} \circ \mathcal{F}_L^{(GFF)} \circ \mathcal{F}_L^{(GSA)} \circ \cdots \circ \mathcal{F}_1^{(GFF)} \circ \mathcal{F}_1^{(GSA)} \circ \mathcal{E}_{in}\right\}.
\end{align*}

A central challenge in understanding the expressivity of Transformers lies in explaining why the self-attention layer effectively captures complex token-wise interactions. Indeed, the self-attention mechanism is the only component within Transformer architectures explicitly designed to integrate token-level information, thus playing a central role in modeling dependencies across sequences. However, the highly nonlinear softmax function commonly used in self-attention presents significant analytical difficulties. Previous research has approached this problem from several perspectives: some studies first analyzed the simpler hardmax function, viewing softmax attention as a smoother approximation whose limiting behavior converges to hardmax \cite{yun2020n, takakura2023approximation}; some restricted the analysis to finite discrete sequences and showed by construction that softmax function acts as a contextual mapping, sending distinct input sequences to distinct output sequences \cite{yun2019transformers, kim2023provable, kajitsuka2023transformers}; yet another approach replaced softmax with more analytically tractable functions, such as piecewise polynomial activation functions \cite{gurevych2022rate, jiao2024convergence, havrilla2024understanding}. In contrast, we provide a fundamentally different proof strategy inspired by the Kolmogorov-Arnold representation theorem (see Proposition \ref{pro: 4}). Our key observation is that representing an arbitrary $d_x n$-dimensional function exactly requires only one inner and one outer function. The inner function in this construction completely separates each entry of the input sequence, effectively simplifying complex token interactions into structured summations. Consequently, we show that the essential role of the self-attention layer can be simplified to performing column-wise summations, providing a more direct and general theoretical justification for the expressive power of Transformer architectures.

The following theorem provides explicit approximation bounds for H\"older continuous functions using generalized Transformer networks. We defer the proof to Section \ref{sec: 3}.

\begin{theorem}\label{thm: 8}
Given $\gamma \in (0,1]$ and $K_\mathcal{H} > 0$, assume that the target function $\bm{F}: [0,1]^{d_x \times n} \rightarrow \mathbb{R}^{d_x \times n}$ satisfies $F_{i,j} \in \mathcal{H}^\gamma ([0,1]^{d_x \times n}, K_\mathcal{H})$ for each $i \in [d_x], j \in [n]$. For any $\varepsilon \in (0,1)$ and $p \in [1,\infty)$, there exists a generalized Transformer network
\begin{align*}
\mathcal{N} \in \mathcal{GT}_{d_x, d_x}(D = 4 d_x n, H = 1, S = d_x, W = 3 d_x n \cdot \lceil\varepsilon^{-\frac{d_x n}{\gamma}}\rceil, L = 6 \lceil\tfrac{1}{\gamma} \log_2 \tfrac{1}{\varepsilon}\rceil)
\end{align*}
such that
\begin{align*}
\|\mathcal{N} - \bm{F}\|_{L^p([0,1]^{d_x \times n})} \leq 4 (d_x n)^3 K_\mathcal{H} \varepsilon.
\end{align*}
\end{theorem}

\section{Discussions and Related Works}\label{sec: 5}

\textbf{Nonparametric regression using neural networks.} The convergence rates of neural network regression estimators have been extensively analyzed in the literature. Minimax optimal rates have been established across various neural network architectures, including under-parameterized sparse deep FNNs \cite{schmidt2020nonparametric, suzuki2018adaptivity}, under-parameterized fully connected deep FNNs \cite{jiao2023deep}, over-parameterized shallow FNNs \cite{yang2024optimal, yang2024nonparametric}, and RNNs \cite{jiao2024approximation}. In contrast, convergence rates for Transformer-based estimators have rarely been observed. \cite{takakura2023approximation} investigated the approximation and estimation capabilities of Transformers as sequence-to-sequence functions operating on infinite-dimensional inputs, where variable-length sliding window attention was considered. Additionally, modifications to the Transformer architecture have been explored, such as replacing the standard softmax function $\sigma_S[\bm{Z}]$ in the self-attention layer by $\bm{Z}\odot\sigma_H[\bm{Z}]$ \cite{gurevych2022rate} and by $\sigma_R[\bm{Z}]$ \cite{havrilla2024understanding}. Although these studies provide insightful constructions and analyses, their avoidance of the standard softmax function does not fully explain the successes observed ever since the introduction of the Transformer mechanism \cite{vaswani2017attention}. Our result (Theorem \ref{thm: 7}) directly addresses this gap by analyzing convergence rates for the standard Transformer architecture explicitly using the original softmax attention mechanism. It has also been shown that neural networks are able to circumvent the curse of dimensionality under certain conditions, for example, when the intrinsic dimension of the regression function is low \cite{nakada2020adaptive, chen2022nonparametric, jiao2023deep, havrilla2024understanding}, or the regression function has certain hierarchical structures \cite{schmidt2020nonparametric, kohler2021rate}. Exploring the conditions under which Transformers similarly mitigate the curse of dimensionality within our framework is an important direction for future research.

\textbf{Assumptions on the smoothness of target function.} In this work, we require the target function to be either H\"older continuous with smoothness index $\gamma \in (0,1]$ or a Sobolev function with bounded first-order weak derivative. It remains an interesting problem whether our methods are adaptive to higher regularity. In the proof of Theorems \ref{thm: 4} and \ref{thm: 5}, we approximate the target function by piecewise constant functions defined on a uniform partition of $[0,1]^{d_x \times n}$ into $K^{d_x n}$ cells. The approximation orders $K^{-\gamma}$ in (\ref{eq: 2}) and $K^{-1}$ in (\ref{eq: 25}) cannot be improved in general. We refer to \cite[Section 6.2]{devore1998nonlinear} for a detailed discussion of saturation and inverse theorems, where certain smoothness properties of a function are deduced from the order of its approximation by multivariate piecewise polynomials. For example, consider a real-valued function $f$ defined on a bounded domain $\Omega \subseteq \mathbb{R}^d$, and let $\Delta$ be a partition of $\Omega$ into a finite number of subsets. Suppose $f$ is approximated by a piecewise constant function
\begin{align*}
s(\bm{x}) = \sum_{\omega \in \Delta} c_\omega \mathbbm{1}_\omega(\bm{x}),
\end{align*}
and assume that $\inf_{s} \|f - s\|_{L^\infty(\Omega)} = o(\operatorname{diam}(\Delta))$ as $\operatorname{diam}(\Delta) := \max_{\omega \in \Delta} \operatorname{diam}(\omega) \rightarrow 0$ for all partitions $\Delta$. We can then easily show that $f$ is a constant function. Indeed, for any $\bm{x}, \bm{y} \in \Omega$, there exists a partition $\Delta$ such that $\bm{x}$ and $\bm{y}$ belong to the same cell $\omega$, where $\operatorname{diam}(\omega) = \operatorname{diam}(\Delta) \leq 2 \|\bm{x} - \bm{y}\|_2$. Let $\bar{s}$ be a best piecewise constant approximation to $f$ under the $L^\infty$-norm (the existence of which is ensured by a compactness argument). Then $|f(\bm{x}) - f(\bm{y})| \leq |f(\bm{x}) - \bar{s}(\bm{x})| + |\bar{s}(\bm{x}) - \bar{s}(\bm{y})| + |\bar{s}(\bm{y}) - f(\bm{y})| \leq 2 \inf_{s} \|f - s\|_{L^\infty(\Omega)}$ since $\bar{s}$ is constant on $\omega$. Hence, by assumption, $|f(\bm{x}) - f(\bm{y})| = o(\operatorname{diam}(\Delta)) = o(\|\bm{x} - \bm{y}\|_2)$ as $\bm{y} \rightarrow \bm{x}$, which implies that $f$ has zero derivative at every point in $\Omega$, i.e., $f$ is constant. For the uniform partition used in our proof, a more subtle analysis is required; see \cite[Chapter 12.2]{devore1993constructive}.

In the proof of Theorem \ref{thm: 8}, we use the Kolmogorov-Arnold representation $f(x_1, \ldots, x_d) = g(3 \sum_{p=1}^d 3^{-p} \phi(x_p))$ for any $d$-variate function $f$ \cite{schmidt2021kolmogorov}. Although this representation allows the transfer of smoothness properties of $f$ to the function $g$ for H\"older continuous $f$ with smoothness index $\gamma \in (0,1]$, it remains unclear whether this representation can be generalized to higher order smoothness or anisotropic smoothness.

\section{Proofs}\label{sec: 6}

Before proceeding, we clarify some simplifications used in the proof.
\begin{enumerate}[itemsep=0em, labelwidth=1em, leftmargin=!]
\item In a self-attention layer, if we set $\bm{W}^{(O)} = \bm{O}$, the layer behaves as an identity mapping due to the presence of the skip connection. Similarly, in a feed-forward layer, setting $\bm{W}^{(2)} = \bm{O}$ causes the layer to degenerate into an identity mapping. Therefore, as long as identity mappings are appropriately introduced, the composition of multiple self-attention layers or feed-forward layers remains consistent with our definition of a Transformer.

\item Since a feed-forward layer applies the same operation to each column of the input matrix, we do not distinguish between matrix and vector inputs when the context is clear, with a slight abuse of notation. For example, given a feed-forward layer $\mathcal{F}^{(FF)}: \mathbb{R}^{D \times n} \rightarrow \mathbb{R}^{D \times n}$ defined as
\begin{align*}
\mathcal{F}^{(FF)}(\boldsymbol{H}) = \boldsymbol{H} + \boldsymbol{W}^{(2)} \sigma_R \left[\boldsymbol{W}^{(1)} \boldsymbol{H} + \boldsymbol{b}^{(1)} \bm{1}_n^\top \right] + \boldsymbol{b}^{(2)} \bm{1}_n^\top,
\end{align*}
we also define $\mathcal{F}^{(FF)}: \mathbb{R}^{D} \rightarrow \mathbb{R}^{D}$ as
\begin{align*}
\mathcal{F}^{(FF)}(\boldsymbol{H}_{:,i}) = \boldsymbol{H}_{:,i} + \boldsymbol{W}^{(2)} \sigma_R \left[\boldsymbol{W}^{(1)} \boldsymbol{H}_{:,i} + \boldsymbol{b}^{(1)}\right] + \boldsymbol{b}^{(2)}
\end{align*}
for each $i$, so that $\mathcal{F}^{(FF)}(\boldsymbol{H}) = (\mathcal{F}^{(FF)}(\boldsymbol{H}_{:,1}), \ldots, \mathcal{F}^{(FF)}(\boldsymbol{H}_{:,n}))$.
\end{enumerate}

If all self-attention layers degenerate into identity mappings, the resulting Transformer reduces to a token-wise ResNet \cite{he2016deep}. We will demonstrate that any token-wise FNN can be represented by a token-wise ResNet, thereby naturally extending existing results on FNNs to Transformers. Specifically, an FNN $\mathcal{N}: \mathbb{R}^{d_x} \rightarrow \mathbb{R}^{d_y}$ is a function that can be parameterized in the form
\begin{align*}
\begin{aligned}
\mathcal{N}_0(\bm{x}) & = \bm{x}, \\
\mathcal{N}_{l+1}(\bm{x}) & = \sigma_R [\bm{A}_{l} \mathcal{N}_{l}(\bm{x}) + \bm{b}_{l}], \quad l=0, \ldots, L-1, \\
\mathcal{N}(\bm{x}) & = \bm{A}_L \mathcal{N}_L(\bm{x}) + \bm{b}_L,
\end{aligned}
\end{align*}
where $\bm{A}_{l} \in \mathbb{R}^{W_{l+1} \times W_{l}}, \boldsymbol{b}_{l} \in \mathbb{R}^{W_{l+1}}$ with $W_0=d_x$, $W_{L+1}=d_y$ and $W_{l}=W$ for $l=1,\ldots,L$. The parameters $W$ and $L$ are referred to as the width and depth of the neural network, respectively. We denote by $\mathcal{FNN}_{d_x, d_y}(W, L)$ the set of functions that can be parameterized in this form with width $W$ and depth $L$.

\begin{lemma}\label{lemma: 10}
Let $d_x, d_y$ be positive integers. For any $\mathcal{N} \in \mathcal{FNN}_{d_x, d_y}(W, L)$ with width $W \geq \max\{d_x, d_y\}$ and depth $L \geq 2$, there exist an embedding map $\mathcal{E}_{in}: \bm{X} \in \mathbb{R}^{d_x \times n} \mapsto \begin{pmatrix} \bm{X} \\ \bm{O} \end{pmatrix} \in \mathbb{R}^{W \times n}$, a projection map $\mathcal{E}_{out}: \begin{pmatrix} \bm{Y} \\ \bm{O} \end{pmatrix} \in \mathbb{R}^{W \times n} \mapsto \bm{Y} \in \mathbb{R}^{d_y \times n}$, and $L$ feed-forward layers with width at most $3W$, such that for any $\bm{X} \in \mathbb{R}^{d_x \times n}$, 
\begin{align*}
\mathcal{E}_{out} \circ \mathcal{F}_L^{(FF)} \circ \cdots \circ \mathcal{F}_1^{(FF)} \circ \mathcal{E}_{in}(\bm{X}) = (\mathcal{N}(\bm{X}_{:,1}), \ldots, \mathcal{N}(\bm{X}_{:,n})) \in \mathbb{R}^{d_y \times n}.
\end{align*}
\end{lemma}

\begin{proof}
The idea is to use the identity $\sigma_R[x] - \sigma_R[-x] = x$ to eliminate the skip connection. For any $\bm{x} \in \mathbb{R}^{d_x}$, direct computation yields 
\begin{align*}
\mathcal{F}_1^{(FF)} \begin{pmatrix} \bm{x} \\ \bm{0} \end{pmatrix} & = \begin{pmatrix} \bm{x} \\ \bm{0} \end{pmatrix} + \begin{pmatrix} \bm{I}_{d_x} & \bm{O} & -\bm{I}_{d_x} & \bm{I}_{d_x} \\ \bm{O} & \bm{I}_{W-d_x} & \bm{O} & \bm{O} \end{pmatrix} \sigma_R\left[\begin{pmatrix}
\bm{A}_0 & \bm{O} \\ 
\bm{I}_{d_x} & \bm{O} \\
-\bm{I}_{d_x} & \bm{O}
\end{pmatrix} \begin{pmatrix} \bm{x} \\ \bm{0} \end{pmatrix} + \begin{pmatrix} \bm{b}_0 \\ \bm{0} \\ \bm{0} \end{pmatrix} \right] \\
& = \sigma_R[\bm{A}_0 \bm{x} + \bm{b}_0] + \begin{pmatrix} 
\bm{x} + -\sigma_R[\bm{x}] + \sigma_R[-\bm{x}] \\ 
\bm{0} 
\end{pmatrix} \\
& = \mathcal{N}_1(\bm{x}),
\end{align*}
where we have used the identity $\sigma_R[\bm{x}] - \sigma_R[-\bm{x}] = \bm{x}$.

Now, assuming that $\mathcal{F}_l^{(FF)} \circ \cdots \circ \mathcal{F}_1^{(FF)} \circ \mathcal{E}_{in}(\bm{x}) = \mathcal{N}_l(\bm{x})$, we define the $(l+1)$-th feed-forward layer $\mathcal{F}_{l+1}^{(FF)}$ as
\begin{align*}
\mathcal{F}_{l+1}^{(FF)}(\mathcal{N}_l(\bm{x})) &= \mathcal{N}_l(\bm{x}) + \begin{pmatrix}
\bm{I}_W, -\bm{I}_W, \bm{I}_W
\end{pmatrix} \sigma_R \left[\begin{pmatrix}
\bm{A}_{l} \\
\bm{I}_W \\
-\bm{I}_W
\end{pmatrix} \mathcal{N}_l(\bm{x}) + \begin{pmatrix}
\bm{b}_{l} \\
\bm{0} \\
\bm{0}
\end{pmatrix} \right] \\
&= \sigma_R \left[\bm{A}_{l} \mathcal{N}_l(\bm{x}) + \bm{b}_{l}\right] = \mathcal{N}_{l+1}(\bm{x}).
\end{align*}
By induction, it follows that
\begin{align*}
\mathcal{F}_{l+1}^{(FF)} \circ \mathcal{F}_l^{(FF)} \circ \cdots \circ \mathcal{F}_1^{(FF)} \circ \mathcal{E}_{in}(\bm{x}) = \mathcal{N}_{l+1}(\bm{x}).
\end{align*}

By the principle of induction, we establish that $\mathcal{F}_{L-1}^{(FF)} \circ \cdots \circ \mathcal{F}_1^{(FF)} \circ \mathcal{E}_{in}(\bm{x}) = \mathcal{N}_{L-1}(\bm{x})$. For the last feed-forward layer, we calculate that
\begin{align*}
&\mathcal{F}_L^{(FF)}(\mathcal{N}_{L-1}(\bm{x})) \\
&= \mathcal{N}_{L-1}(\bm{x}) + \begin{pmatrix}
\bm{A}_{L} & -\bm{I}_{d_y} & \bm{O} & \bm{I}_{d_y} & \bm{O} \\
\bm{O} & \bm{O} & -\bm{I}_{W-d_y} & \bm{O} & \bm{I}_{W-d_y}
\end{pmatrix} \sigma_R\left[\begin{pmatrix}
\bm{A}_{L-1} \\
\bm{I}_W \\
-\bm{I}_W
\end{pmatrix} \mathcal{N}_{L-1}(\bm{x}) + \begin{pmatrix}
\bm{b}_{L-1} \\
\bm{0} \\
\bm{0}
\end{pmatrix} \right] + \begin{pmatrix}
\bm{b}_L \\
\bm{0}
\end{pmatrix} \\
&= \begin{pmatrix}
\bm{A}_{L} \sigma_R [\bm{A}_{L-1} \mathcal{N}_{L-1}(\bm{x}) + \bm{b}_{L-1}] + \bm{b}_L \\
\bm{0}
\end{pmatrix} \\
&= \begin{pmatrix}
\mathcal{N}(\bm{x}) \\
\bm{0}
\end{pmatrix}.
\end{align*}
Thus, we obtain
\begin{align*}
\mathcal{E}_{out} \circ \mathcal{F}_L^{(FF)} \circ \cdots \circ \mathcal{F}_1^{(FF)} \circ \mathcal{E}_{in}(\bm{x}) = \mathcal{N}(\bm{x}).
\end{align*}
Since each feed-forward layer in our construction has width at most $3W$, the proof is complete by considering $\bm{x} = \bm{X}_{:,i}$ for $i \in [n]$.
\end{proof}

The following proposition gives basic properties of Transformer networks that enable the recursive construction of complex architectures.
\begin{proposition}\label{pro: 1}
Let $\mathcal{N}_i \in \mathcal{T}_{d_i, k_i}(D_i, H_i, S_i, W_i, L_i)$ for $i=1, 2$.
\begin{enumerate}[itemsep=0em, labelwidth=1em, leftmargin=!]
\item If $d_1 = d_2$, $k_1 = k_2$, and $D_1 \leq D_2, H_1 \leq H_2, S_1 \leq S_2, W_1 \leq W_2, L_1 \leq L_2$, then 
\begin{align*}
\mathcal{T}_{d_1, k_1}(D_1, H_1, S_1, W_1, L_1) \subseteq \mathcal{T}_{d_2, k_2}(D_2, H_2, S_2, W_2, L_2).
\end{align*}

\item (Concatenation) If define $\mathcal{N} \begin{pmatrix} \bm{X} \\ \bm{Y} \end{pmatrix} = \begin{pmatrix} \mathcal{N}_1 (\bm{X}) \\ \mathcal{N}_2 (\bm{Y}) \end{pmatrix}$, then 
\begin{align*}
\mathcal{N} \in \mathcal{T}_{d_1 + d_2, k_1 + k_2} (D_1 + D_2, H_1 + H_2, \max\{S_1, S_2\}, W_1 + W_2, \max\{L_1, L_2\}).
\end{align*}

\item (Summation) If $d_1 = d_2$ and $k_1 = k_2$, then 
\begin{align*}
\mathcal{N}_1 + \mathcal{N}_2 \in \mathcal{T}_{d_1, k_1} (D_1 + D_2, H_1 + H_2, \max\{S_1, S_2\}, W_1 + W_2, \max\{L_1, L_2\}).
\end{align*}
\end{enumerate}
\end{proposition}

\begin{proof}
We provide the proof for (2), as the arguments for (1) and (3) follow analogously.

Let
\begin{align*}
    \mathcal{N}_i = \mathcal{E}_{i, out} \circ \mathcal{F}_{i, L_i}^{(FF)} \circ \mathcal{F}_{i, L_i}^{(SA)} \circ \cdots \circ \mathcal{F}_{i, 1}^{(FF)} \circ \mathcal{F}_{i, 1}^{(SA)} \circ \mathcal{E}_{i, in} \in \mathcal{T}_{d_i, k_i}(D_i, H_i, S_i, W_i, L_i), \quad i = 1,2.
\end{align*}
Without loss of generality, assume that $L_1 = L_2 = L$, since the identity mapping can be viewed as a special self-attention layer or a feed-forward layer. We define the following components:
\begin{enumerate}[itemsep=0em, labelwidth=1em, leftmargin=!]
\item Input Embedding:
\begin{align*}
\mathcal{E}_{in} \begin{pmatrix} \bm{X} \\ \bm{Y} \end{pmatrix} = \begin{pmatrix}
\bm{E}_{1,in} & \\
& \bm{E}_{2,in}
\end{pmatrix} \begin{pmatrix} \bm{X} \\ \bm{Y} \end{pmatrix} + \begin{pmatrix} \bm{P}_1 \\ \bm{P}_2 \end{pmatrix} = \begin{pmatrix} \bm{E}_{1,in} \bm{X} + \bm{P}_1 \\ \bm{E}_{2,in} \bm{Y} + \bm{P}_2 \end{pmatrix}  = \begin{pmatrix} \mathcal{E}_{1,in}(\bm{X}) \\ \mathcal{E}_{2,in}(\bm{Y}) \end{pmatrix}
\end{align*}

\item Feed-forward Layer:
\begin{align*}
\begin{aligned}
& \mathcal{F}_{l}^{(FF)} \begin{pmatrix} \bm{X} \\ \bm{Y} \end{pmatrix} \\
&= \begin{pmatrix} \bm{X} \\ \bm{Y} \end{pmatrix} + \begin{pmatrix}
\bm{W}_{1,l}^{(2)} & \\
& \bm{W}_{2,l}^{(2)}
\end{pmatrix} \sigma_R \left[ \begin{pmatrix}
\bm{W}_{1,l}^{(1)} & \\
& \bm{W}_{2,l}^{(1)}
\end{pmatrix} \begin{pmatrix} \bm{X} \\ \bm{Y} \end{pmatrix} + \begin{pmatrix} \bm{b}_{1,l}^{(1)} \\ \bm{b}_{2,l}^{(1)} \end{pmatrix} \bm{1}_n^\top \right] + \begin{pmatrix} \bm{b}_{1,l}^{(2)} \\ \bm{b}_{2,l}^{(2)} \end{pmatrix} \bm{1}_n^\top \\
&= \begin{pmatrix}
\bm{X} + \bm{W}_{1,l}^{(2)} \sigma_R [\bm{W}_{1,l}^{(1)} \bm{X} + \bm{b}_{1,l}^{(1)} \bm{1}_n^\top] + \bm{b}_{1,l}^{(2)} \bm{1}_n^\top \\
\bm{Y} + \bm{W}_{2,l}^{(2)} \sigma_R [\bm{W}_{2,l}^{(1)} \bm{Y} + \bm{b}_{2,l}^{(1)} \bm{1}_n^\top] + \bm{b}_{2,l}^{(2)} \bm{1}_n^\top
\end{pmatrix} \\
&= \begin{pmatrix}
\mathcal{F}_{1,l}^{(FF)} (\bm{X}) \\
\mathcal{F}_{2,l}^{(FF)} (\bm{Y})
\end{pmatrix}
\end{aligned}
\end{align*}

\item Self-Attention Layer:
\begin{align*}
\begin{aligned}
& \mathcal{F}_l^{(SA)} \begin{pmatrix} \bm{X} \\ \bm{Y} \end{pmatrix} \\
&= \begin{pmatrix} \bm{X} \\ \bm{Y} \end{pmatrix} + \sum_{h=1}^{H_1} \begin{pmatrix}
\boldsymbol{W}_{1,h,l}^{(O)} \\
\bm{O}
\end{pmatrix} \left(\boldsymbol{W}_{1,h,l}^{(V)}, \bm{O}\right) \begin{pmatrix} \bm{X} \\ \bm{Y} \end{pmatrix} \sigma_S \left[ \left( \left( \boldsymbol{W}_{1,h,l}^{(K)}, \bm{O} \right) \begin{pmatrix} \bm{X} \\ \bm{Y} \end{pmatrix} \right)^\top \left( \boldsymbol{W}_{1,h,l}^{(Q)}, \bm{O} \right) \begin{pmatrix} \bm{X} \\ \bm{Y} \end{pmatrix} \right] \\
&~~~ + \sum_{h=1}^{H_2} \begin{pmatrix}
\bm{O} \\
\boldsymbol{W}_{2,h,l}^{(O)}
\end{pmatrix} \left( \bm{O}, \boldsymbol{W}_{2,h,l}^{(V)}\right) \begin{pmatrix} \bm{X} \\ \bm{Y} \end{pmatrix} \sigma_S \left[ \left( \left( \bm{O}, \boldsymbol{W}_{2,h,l}^{(K)} \right) \begin{pmatrix} \bm{X} \\ \bm{Y} \end{pmatrix} \right)^\top \left( \bm{O}, \boldsymbol{W}_{2,h,l}^{(Q)} \right) \begin{pmatrix} \bm{X} \\ \bm{Y} \end{pmatrix} \right] \\
& = \begin{pmatrix}
\boldsymbol{X} + \sum_{h=1}^{H_1} \boldsymbol{W}_{1,h,l}^{(O)} \left(\boldsymbol{W}_{1,h,l}^{(V)} \boldsymbol{X}\right) \sigma_S \left[\left(\boldsymbol{W}_{1,h,l}^{(K)} \boldsymbol{X}\right)^{\top} \left(\boldsymbol{W}_{1,h,l}^{(Q)} \boldsymbol{X}\right)\right] \\
\boldsymbol{Y} + \sum_{h=1}^{H_2} \boldsymbol{W}_{2,h,l}^{(O)} \left(\boldsymbol{W}_{2,h,l}^{(V)} \boldsymbol{Y}\right) \sigma_S \left[\left(\boldsymbol{W}_{2,h,l}^{(K)} \boldsymbol{Y}\right)^{\top} \left(\boldsymbol{W}_{2,h,l}^{(Q)} \boldsymbol{Y}\right)\right] 
\end{pmatrix} \\
& = \begin{pmatrix}
\mathcal{F}_{1,l}^{(SA)} (\bm{X}) \\
\mathcal{F}_{2,l}^{(SA)} (\bm{Y})
\end{pmatrix}
\end{aligned}
\end{align*}

\item Output Projection:
\begin{align*}
\mathcal{E}_{out} \begin{pmatrix} \bm{X} \\ \bm{Y} \end{pmatrix} = \begin{pmatrix} 
\bm{E}_{1,out} \\ 
\bm{E}_{2,out} 
\end{pmatrix} \begin{pmatrix} \bm{X} \\ \bm{Y} \end{pmatrix} = \begin{pmatrix} \bm{E}_{1,out} \bm{X} \\ \bm{E}_{2,out} \bm{Y} \end{pmatrix} = \begin{pmatrix} \mathcal{E}_{1,out} (\bm{X}) \\ \mathcal{E}_{2,out} (\bm{Y}) \end{pmatrix}
\end{align*}
\end{enumerate}
By direct verification, we obtain
\begin{align*}
\mathcal{N} \begin{pmatrix} \bm{X} \\ \bm{Y} \end{pmatrix} := \mathcal{E}_{out} \circ \mathcal{F}_L^{(FF)} \circ \mathcal{F}_L^{(SA)} \circ \cdots \circ \mathcal{F}_1^{(FF)} \circ \mathcal{F}_1^{(SA)} \circ \mathcal{E}_{in} \begin{pmatrix} \bm{X} \\ \bm{Y} \end{pmatrix} = \begin{pmatrix} \mathcal{N}_1 (\bm{X}) \\ \mathcal{N}_2 (\bm{Y}) \end{pmatrix}.
\end{align*}
Furthermore, it follows that $\mathcal{N} \in \mathcal{T}_{d_1+d_2, k_1+k_2} (D_1+D_2, H_1+H_2, \max\{S_1,S_2\}, W_1+W_2, L)$, thus completing the proof.
\end{proof}

\subsection{Proof of Theorems \ref{thm: 4} and \ref{thm: 5}}\label{sec: 1}

Given $K \in \mathbb{N}$ and $\delta \in (0, \frac{1}{K})$, we define a trifling region $\Omega([0,1]^{D \times n}, K, \delta) \subseteq [0,1]^{D \times n}$ as
\begin{align*}
\Omega([0,1]^{D \times n}, K, \delta) := \left\{\bm{X} \in [0,1]^{D \times n}: \exists X_{i,j} \in \cup_{t=1}^{K-1} (\tfrac{t}{K}, \tfrac{t}{K} + \delta)\right\}.
\end{align*}
The introduction of the trifling region serves to identify the "bad" areas where mismatches occur when approximating a discontinuous multi-step function using continuous piecewise linear functions, which can be implemented by a feed-forward layer. Since the trifling region has arbitrarily small Lebesgue measure, we focus on function approximation in the "good" region, namely, the complement domain $[0,1]^{D \times n} \setminus \Omega([0,1]^{D \times n}, K, \delta)$.

\begin{proposition}\label{pro: 2}
Given $\gamma \in (0,1]$ and $K_\mathcal{H} > 0$, assume that $\bm{F}: [0,1]^{d_x \times n} \rightarrow \mathbb{R}^{d_x \times n}$ satisfies $F_{i,j} \in \mathcal{H}^\gamma ([0,1]^{d_x \times n}, K_\mathcal{H})$ for each $i \in [d_x], j \in [n]$. For any $K \in \mathbb{N}$ and $\delta \in (0,\tfrac{1}{K})$, there exists a Transformer network $\mathcal{N} \in \mathcal{T}_{d_x, d_x}(d_x, 1, 1, 5 n K^{d_x n}, 2)$ such that
\begin{enumerate}[itemsep=0em, labelwidth=1em, leftmargin=!]
\item $|\mathcal{N}_{i,j}(\bm{X}) - F_{i,j}(\bm{X})| \leq K_{\mathcal{H}} (d_x n)^{\gamma/2} K^{-\gamma}$ for any $i \in [d_x]$, $j \in [n]$ and $\bm{X} \in [0,1]^{d_x \times n} \setminus \Omega([0,1]^{d_x \times n}, K, \delta)$,
\item $\|\mathcal{N}(\bm{X})\|_F \leq \sqrt{d_x n} K_{\mathcal{H}}$ for any $\bm{X} \in \mathbb{R}^{d_x \times n}$.
\end{enumerate}
\end{proposition}

\begin{proof}
We basically follow the proof of \cite[Proposition 1]{kajitsuka2023transformers}.

\textbf{Step 1:} We begin by uniformly partitioning the domain $[0,1]^{d_x \times n}$ into $K^{d_x n}$ subregions and constructing a piecewise constant function $\overline{\bm{F}}$ that approximates the target function $\bm{F}$, with an approximation error scales as $K^{-\gamma}$. Specifically, let $K \in \mathbb{N}$ denote the granularity of the grid
\begin{align*}
\mathbb{G}_K = \left\{\tfrac{1}{K}, \tfrac{2}{K}, \dots, 1\right\}^{d_x \times n}.
\end{align*}
We define each subregion as
\begin{align*}
\omega_{\bm{G}} := \prod_{i \in [d_x], j \in [n]}
{\footnotesize\begin{cases}
[G_{i,j} - \tfrac{1}{K}, G_{i,j}], & \text{if } G_{i,j} = \tfrac{1}{K} \\
(G_{i,j} - \tfrac{1}{K}, G_{i,j}], & \text{otherwise}
\end{cases}}
\end{align*}
associated with $\bm{G} \in \mathbb{G}_K$. Clearly, these subregions form a partition of the domain $[0,1]^{d_x \times n} = \bigcup_{\bm{G} \in \mathbb{G}_K} \omega_{\bm{G}}$. Given a target function $\bm{F}$ with $F_{i,j} \in \mathcal{H}^\gamma([0,1]^{d_x \times n}, K_\mathcal{H})$, we define a piecewise constant approximation of $\bm{F}$ as
\begin{align*}
\overline{\bm{F}}(\bm{X}) = \sum_{\bm{G} \in \mathbb{G}_K} \bm{F}(\bm{G}) \mathbbm{1}_{\omega_{\bm{G}}}(\bm{X}),
\end{align*}
where $\mathbbm{1}_{\omega}$ denotes the indicator function of set $\omega$. That is, within each subregion $\omega_{\bm{G}}$, we approximate $\bm{F}$ using its value at the grid point $\bm{G}$. By the regularity of $\bm{F}$, we have the error estimate
\begin{align}\label{eq: 2}
\begin{aligned}
|F_{i,j}(\bm{X}) - \overline{F}_{i,j}(\bm{X})| &= \left|\sum_{\bm{G} \in \mathbb{G}_K} (F_{i,j}(\bm{X}) - \overline{F}_{i,j}(\bm{X})) \mathbbm{1}_{\omega_{\bm{G}}}(\bm{X})\right| \\
&= \left|\sum_{\bm{G} \in \mathbb{G}_K} (F_{i,j}(\bm{X}) - F_{i,j}(\bm{G})) \mathbbm{1}_{\omega_{\bm{G}}}(\bm{X})\right| \\
&\leq \sum_{\bm{G} \in \mathbb{G}_K} |(F_{i,j}(\bm{X}) - F_{i,j}(\bm{G}))| \mathbbm{1}_{\omega_{\bm{G}}}(\bm{X}) \\
&\leq \sum_{\bm{G} \in \mathbb{G}_K} K_{\mathcal{H}} \|\bm{X} - \bm{G}\|_F^\gamma \mathbbm{1}_{\omega_{\bm{G}}}(\bm{X}) \\
&\leq K_{\mathcal{H}} (d_x n)^{\gamma/2} K^{-\gamma} \sum_{\bm{G} \in \mathbb{G}_K} \mathbbm{1}_{\omega_{\bm{G}}}(\bm{X}) \\
&= K_{\mathcal{H}} (d_x n)^{\gamma/2} K^{-\gamma},
\end{aligned}
\end{align}
for any $i \in [d_x]$, $j \in [n]$ and $\bm{X} \in [0,1]^{d_x \times n}$.

\textbf{Step 2:} Given a positional encoding matrix $\bm{P}$ and a spatial discretization function $\overline{\mathcal{F}}_1^{(FF)}$ satisfying 
\begin{align*}
\overline{\mathcal{F}}_1^{(FF)}(\bm{X}+\bm{P}) = \bm{G} + \bm{P}, \quad \text{for all } \bm{X} \in \omega_{\bm{G}},
\end{align*}
our objective is to construct a feed-forward layer $\mathcal{F}_1^{(FF)}$, with width at most $2 n d_x (K+1)$, that accurately represents $\overline{\mathcal{F}}_1^{(FF)}$ outside the trifling region $\Omega([0,1]^{d_x \times n}, K, \delta)$, that is,  
\begin{align*}
\mathcal{F}_1^{(FF)}(\bm{X}+\bm{P}) = \overline{\mathcal{F}}_1^{(FF)}(\bm{X}+\bm{P}), \quad \text{for any } \bm{X} \in [0,1]^{d_x \times n} \setminus \Omega([0,1]^{d_x \times n}, K, \delta).
\end{align*}
To achieve this goal, we first approximate a univariate multiple-step function using a piecewise linear function, and then extend this function to matrix elements by stacking.

We define the embedding layer as  
\begin{align*}
\mathcal{E}_{in} (\bm{X}) = \bm{X} + \bm{P} \in \mathbb{R}^{d_x \times n},
\end{align*}
where the positional encoding matrix $\bm{P}$ is chosen as  
\begin{align*}
\bm{P} = \begin{pmatrix}
0 & 2 & \cdots & 2(n-1) \\
\vdots & \vdots & & \vdots \\
0 & 2 & \cdots & 2(n-1)
\end{pmatrix}.
\end{align*}
Since $\bm{X} \in [0,1]^{d_x \times n}$, the positional encoding ensures that the columns of $\bm{X} + \bm{P}$ are mapped to distinct intervals, that is, $[\bm{X} + \bm{P}]_{i,j} \in [2j - 2, 2j - 1]$ for each $j \in [n]$. Now, consider a multiple-step function $\text{step}_K(z)$ defined on $[0,1]$ as
\begin{align*}
\text{step}_K (z) =
\begin{cases}
\tfrac{1}{K}, & 0 \leq z \leq \tfrac{1}{K} \\
\tfrac{2}{K}, & \tfrac{1}{K} < z \leq \tfrac{2}{K} \\
\tfrac{3}{K}, & \tfrac{2}{K} < z \leq \tfrac{3}{K} \\
\vdots & \vdots \\
1, & 1-\tfrac{1}{K} < z \leq 1
\end{cases}.
\end{align*}
Given $\delta \in (0, \tfrac{1}{K})$, by translations, scalings and summations of the $\delta$-approximated step function
\begin{align*}
\sigma_R[z / \delta]-\sigma_R[z / \delta-1] = \begin{cases}
0 & z \leq 0 \\ 
z / \delta & 0 < z < \delta \\ 
1 & \delta \leq z
\end{cases},
\end{align*}
we define
\begin{align*}
f(z) =& \frac{1}{K} + \sum_{j=1}^{n} \sum_{t=1}^{K-1} \frac{1}{K} \left( \sigma_{R} \left[\frac{z-2(j-1)}{\delta} - \frac{t}{\delta K}\right] - \sigma_{R} \left[\frac{z-2(j-1)}{\delta} - 1 - \frac{t}{\delta K}\right]\right) \\
& + \sum_{j=1}^{n-1} \left(1+\frac{1}{K}\right) \left(\sigma_{R} [z-(2j-1)] - \sigma_{R} [z-2j]\right).
\end{align*}
It is straightforward to verify that $f(z + (2j-2)) = \text{step}_K (z) + (2j-2)$ for all $z \in [0,1] \setminus \Omega([0,1], K, \delta)$ and $j \in [n]$.  Moreover, the function $f$ can be represented by a shallow ReLU network with $2 n K - 2$ units in the hidden layer.

We then concatenate multiple $f$ in parallel to construct a feed-forward layer $\mathcal{F}_1^{(FF)}: \mathbb{R}^{d_x \times n} \rightarrow \mathbb{R}^{d_x \times n}$ with width at most $2 n d_x (K+1)$, satisfying
\begin{align*}
\mathcal{F}_1^{(FF)}(\bm{X}+\bm{P}) &= \begin{pmatrix}
f(X_{1,1}+P_{1,1}) & \cdots & f(X_{1,n}+P_{1,n}) \\
\vdots & \ddots & \vdots \\
f(X_{d_x,1}+P_{d_x,1}) & \cdots & f(X_{d_x,n}+P_{d_x,n}) \\
\end{pmatrix} \\
&\approx \begin{pmatrix}
\text{step}_K (X_{1,1}) + P_{1,1} & \cdots & \text{step}_K (X_{1,n}) + P_{1,n} \\
\vdots & \ddots & \vdots \\
\text{step}_K (X_{d_x,1}) + P_{d_x,1} & \cdots & \text{step}_K (X_{d_x,n}) + P_{d_x,n} \\
\end{pmatrix} \\
&= \begin{pmatrix}
\text{step}_K (X_{1,1}) & \cdots & \text{step}_K (X_{1,n}) \\
\vdots & \ddots & \vdots \\
\text{step}_K (X_{d_x,1}) & \cdots & \text{step}_K (X_{d_x,n}) \\
\end{pmatrix} + \bm{P}.
\end{align*}
Noting that
\begin{align*}
\Omega([0,1]^{d_x \times n}, K, \delta) = \bigcup_{i \in [d_x], j \in [n]} \{\bm{X}: X_{i,j} \in \Omega([0,1], K, \delta)\},
\end{align*}
we conclude that $\mathcal{F}_1^{(FF)}(\bm{X}+\bm{P}) = \overline{\mathcal{F}}_1^{(FF)}(\bm{X}+\bm{P})$ for any $\bm{X} \in [0,1]^{d_x \times n} \setminus \Omega([0,1]^{d_x \times n}, K, \delta)$.

\textbf{Step 3:} Since $\{\bm{G} + \bm{P}: \bm{G} \in \mathbb{G}_K\}$ can be regarded as sequences, each of which has no duplicate token due to positional encoding, it follows from \cite[Theorem 2]{kajitsuka2023transformers} that there exists a self-attention layer $\mathcal{F}^{(SA)}: \mathbb{R}^{d_x \times n} \rightarrow \mathbb{R}^{d_x \times n}$ with $H = 1$ and $s = 1$ that serves as a contextual mapping for such input sequences (see \cite{kajitsuka2023transformers, yun2019transformers} for further discussion). In essence, a contextual mapping is a bijection between sequences that satisfies $\mathcal{F}^{(SA)}(\bm{G}^{(i)} + \bm{P})_{:,k} \neq \mathcal{F}^{(SA)}(\bm{G}^{(j)} + \bm{P})_{:,l}$ if $\bm{G}^{(i)} \neq \bm{G}^{(j)} \in \mathbb{G}_K$ or $k \neq l \in [n]$. The remaining is to associate each output token with its corresponding function value using a feed-forward layer, which reduces to a memorization task. Lemma \ref{lemma: 1} gives a construction of such a feed-forward layer, denoted as $\mathcal{F}_2^{(FF)}$, with width at most $5 n K^{d_x n}$ (set $r = n \cdot |\mathbb{G}_K| \leq n 
 K^{d_x n}$ therein), such that
\begin{align*}
\mathcal{F}_2^{(FF)}(\mathcal{F}^{(SA)}(\bm{G}+\bm{P})) = \bm{F}(\bm{G}) \quad \text{for all } \bm{G} \in \mathbb{G}_K,
\end{align*}
and $\|\mathcal{F}_2^{(FF)}(\bm{Z})\|_F \leq \sqrt{d_x n} K_{\mathcal{H}}$.

Let $\mathcal{E}_{out}$ be the identity mapping. It holds that $\mathcal{E}_{out} \circ \mathcal{F}_2^{(FF)} \circ \mathcal{F}^{(SA)} \circ \mathcal{F}_1^{(FF)} \circ \mathcal{E}_{in} \in \mathcal{T}_{d_x, d_x}(d_x, 1, 1, 5 n K^{d_x n}, 2)$. Note that for any $\bm{X} \in \omega_{\bm{G}} \setminus \Omega([0,1]^{d_x \times n}, K, \delta)$,
\begin{align*}
& \mathcal{E}_{out} \circ \mathcal{F}_2^{(FF)} \circ \mathcal{F}^{(SA)} \circ \mathcal{F}_1^{(FF)} \circ \mathcal{E}_{in}(\bm{X}) \\
&= \mathcal{F}_2^{(FF)} \circ \mathcal{F}^{(SA)} \circ \mathcal{F}_1^{(FF)}(\bm{X} + \bm{P}) \\
&= \mathcal{F}_2^{(FF)} \circ \mathcal{F}^{(SA)} \circ \overline{\mathcal{F}}_1^{(FF)}(\bm{X} + \bm{P}) \\
&= \mathcal{F}_2^{(FF)} \circ \mathcal{F}^{(SA)} (\bm{G} + \bm{P}) \\
&= \bm{F}(\bm{G}) = \overline{\bm{F}}(\bm{X}).
\end{align*}
Thus, for any $\bm{X} \in [0,1]^{d_x \times n} \setminus \Omega([0,1]^{d_x \times n}, K, \delta) = \bigcup_{\bm{G} \in \mathbb{G}_K} \omega_{\bm{G}} \setminus \Omega([0,1]^{d_x \times n}, K, \delta)$, we have
\begin{align*}
\mathcal{E}_{out} \circ \mathcal{F}_2^{(FF)} \circ \mathcal{F}^{(SA)} \circ \mathcal{F}_1^{(FF)} \circ \mathcal{E}_{in}(\bm{X}) = \overline{\bm{F}}(\bm{X}),
\end{align*}
which completes the proof by noting (\ref{eq: 2}).

\end{proof}

\begin{proposition}\label{pro: 3}
Given $\gamma \in (0,1]$ and $K_\mathcal{H} > 0$, assume that $\bm{F}: [0,1]^{d_x \times n} \rightarrow \mathbb{R}^{d_x \times n}$ with each entry $F_{i,j} \in \mathcal{H}^\gamma ([0,1]^{d_x \times n}, K_\mathcal{H})$. For any $\varepsilon > 0$, $K \in \mathbb{N}$ and $\delta \in (0, \frac{1}{3 K}]$, if $\widetilde{\mathcal{N}} \in \mathcal{T}_{d_x, d_x}(D, H, S, W, L)$ is a Transformer network that satisfies
\begin{align*}
|\widetilde{\mathcal{N}}_{i,j}(\bm{X}) - F_{i,j}(\bm{X})| \leq \varepsilon
\end{align*}
for any $i \in [d_x]$, $j \in [n]$ and $\bm{X} \in [0,1]^{d_x \times n} \setminus \Omega([0,1]^{d_x \times n}, K, \delta)$, then there exists a new Transformer network
\begin{align*}
\mathcal{N} \in \mathcal{T}_{d_x, d_x}(3^{d_x n} \max\{D, 5 d_x\}, 3^{d_x n} H, S, 3^{d_x n} \max\{W, 14 d_x\}, L + 2 d_x n),
\end{align*}
such that
\begin{align*}
|\mathcal{N}_{i,j}(\bm{X}) - F_{i,j}(\bm{X})| \leq \varepsilon + d_x n K_{\mathcal{H}} \delta^{\gamma}
\end{align*}
for any $i \in [d_x]$, $j \in [n]$ and $\bm{X} \in [0,1]^{d_x \times n}$.
\end{proposition}

\begin{proof}
We basically follow the proof of \cite[Theorem 2.1]{lu2021deep}.

\textbf{Step 1:} We prove that, given $i \in [d_x]$, $j \in [n]$, $F_{i,j} \in \mathcal{H}^\gamma ([0,1]^{d_x \times n}, K_\mathcal{H})$, and a general function $G_{i,j}: \mathbb{R}^{d_x \times n} \rightarrow \mathbb{R}$ satisfying
\begin{align}\label{eq: 9}
|G_{i,j}(\bm{X}) - F_{i,j}(\bm{X})| \leq \varepsilon \ \text{ for any } \bm{X} \in [0,1]^{d_x \times n} \setminus \Omega([0,1]^{d_x \times n}, K, \delta),
\end{align}
then
\begin{align}\label{eq: 10}
|\Phi_{i,j}(\bm{X}) - F_{i,j}(\bm{X})| \leq \varepsilon + d_x n K_{\mathcal{H}} \delta^{\gamma} \ \text{ for any } \bm{X} \in [0,1]^{d_x \times n},
\end{align}
where $\Phi_{i,j} := \Phi_{i,j}^{(d_x n)}$ is defined by induction through
\begin{align}\label{eq: 12}
\Phi_{i,j}^{(k)}(\bm{X}) := \operatorname{mid}\left(\Phi_{i,j}^{(k-1)}\left(\bm{X} - \delta \bm{E}^{(k)}\right), \Phi_{i,j}^{(k-1)}\left(\bm{X}\right), \Phi_{i,j}^{(k-1)}\left(\bm{X} + \delta \bm{E}^{(k)}\right)\right)
\end{align}
for $k = 1, 2, \cdots, d_x n$, $\Phi_{i,j}^{(0)} = G_{i,j}$, $\operatorname{mid}(\cdot,\cdot,\cdot)$ is a function returning the middle value of three inputs, and $\bm{E}^{(u + (v-1) d_x)}$ denotes the matrix with 1 at the $(u,v)$-th position and 0 elsewhere for $u \in [d_x], v \in [n]$. In other words, if $G_{i,j}$ provides a uniform approximation outside the trifling region, then the carefully constructed $\Phi_{i,j}$ extends this uniform approximation to the entire domain, with only a slight increase in the approximation error.

Note that $\{\bm{E}^{(k)}\}_{k=1}^{d_x n}$ defined above is a re-indexing of the standard basis in $\mathbb{R}^{d_x \times n}$. We re-index the elements of $\bm{X} = (X_{u,v})$ in the same manner. Let $X^{(u + (v-1) d_x)} = X_{u,v}$ for $u \in [d_x], v \in [n]$. Using this notation, define
\begin{align*}
\Omega_{k} := \left\{\bm{X}: X^{(i)} \in {\footnotesize\begin{cases}
[0,1], & \text{ if } i \leq k \\
[0,1] \setminus \Omega([0,1], K, \delta), & \text{ if } i > k
\end{cases}} \right\}.
\end{align*}
Clearly, $\Omega_0 = [0,1]^{d_x \times n} \setminus \Omega([0,1]^{d_x \times n}, K, \delta)$ and $\Omega_{d_x n} = [0,1]^{d_x \times n}$.

We will prove by induction that for each $k \in \{0, 1, \ldots, d_x n\}$,
\begin{align}\label{eq: 11}
|\Phi_{i,j}^{(k)}(\bm{X}) - F_{i,j}(\bm{X})| \leq \varepsilon + k \cdot K_{\mathcal{H}} \delta^{\gamma} \ \text{ for any } \bm{X} \in \Omega_{k}.
\end{align}
As the final step of the induction, we derive
\begin{align*}
|\Phi_{i,j}(\bm{X}) - F_{i,j}(\bm{X})| &= |\Phi_{i,j}^{(d_x n)}(\bm{X}) - F_{i,j}(\bm{X})| \\
& \leq \varepsilon + d_x n K_{\mathcal{H}} \delta^{\gamma} \ \text{ for any } \bm{X} \in \Omega_{d_x n} = [0,1]^{d_x \times n},
\end{align*}
which completes the proof of (\ref{eq: 10}).

In the base case, it follows from (\ref{eq: 9}) that
\begin{align*}
|\Phi_{i,j}^{(0)}(\bm{X}) - F_{i,j}(\bm{X})| &= |G_{i,j}(\bm{X}) - F_{i,j}(\bm{X})| \\
& \leq \varepsilon \ \text{ for any } \bm{X} \in \Omega_0 = [0,1]^{d_x \times n} \setminus \Omega([0,1]^{d_x \times n}, K, \delta).
\end{align*}
Now, assume that for some $k \in \{1, 2, \dots, d_x n\}$,
\begin{align*}
|\Phi_{i,j}^{(k-1)}(\bm{X}) - F_{i,j}(\bm{X})| \leq \varepsilon + (k-1) K_{\mathcal{H}} \delta^{\gamma} \ \text{ for any } \bm{X} \in \Omega_{k-1}.
\end{align*}
For fixed $X^{(1)}, \ldots, X^{(k-1)} \in [0,1]$ and $X^{(k+1)}, \ldots, X^{(d_x n)} \in [0,1] \setminus \Omega([0,1], K, \delta)$, define
\begin{align*}
\phi(t) = \Phi_{i,j}^{(k-1)}(X^{(1)}, \ldots, X^{(k-1)}, t, X^{(k+1)}, \ldots, X^{(d_x n)})
\end{align*}
and
\begin{align*}
f(t) = F_{i,j}(X^{(1)}, \ldots, X^{(k-1)}, t, X^{(k+1)}, \ldots, X^{(d_x n)}).
\end{align*}
The induction hypothesis gives
\begin{align*}
|\phi(t) - f(t)| \leq \varepsilon + (k-1) \cdot K_{\mathcal{H}} \delta^{\gamma} \ \text{ for any } t \in [0,1] \setminus \Omega([0,1], K, \delta).
\end{align*}
Since $F_{i,j} \in \mathcal{H}^\gamma ([0,1]^{d_x \times n}, K_\mathcal{H})$ implies $f \in \mathcal{H}^{\gamma} ([0,1], K_{\mathcal{H}})$, applying Lemma \ref{lemma: 5} to the univariate functions $\phi(t)$ and $f(t)$ yields
\begin{align*}
|\widetilde{\phi}(t)-f(t)| \leq \varepsilon + (k-1) \cdot K_{\mathcal{H}} \delta^{\gamma} + K_{\mathcal{H}} \delta^{\gamma} = \varepsilon + k \cdot K_{\mathcal{H}} \delta^{\gamma} \ \text{ for any } t \in [0,1],
\end{align*}
where
\begin{align*}
\widetilde{\phi}(t) & =  \operatorname{mid}\left(\phi(t-\delta), \phi(t), \phi(t+\delta)\right) \\
& \begin{aligned} 
= \operatorname{mid}( & \Phi_{i,j}^{(k-1)}(X^{(1)}, \ldots, X^{(k-1)}, t-\delta, X^{(k+1)}, \ldots, X^{(d_x n)}),  \\
& \Phi_{i,j}^{(k-1)}(X^{(1)}, \ldots, X^{(k-1)}, t, X^{(k+1)}, \ldots, X^{(d_x n)}), \\
& \Phi_{i,j}^{(k-1)}(X^{(1)}, \ldots, X^{(k-1)}, t+\delta, X^{(k+1)}, \ldots, X^{(d_x n)})
) \\
\end{aligned} \\
& = \Phi_{i,j}^{(k)}(X^{(1)}, \ldots, X^{(k-1)}, t, X^{(k+1)}, \ldots, X^{(d_x n)}) 
\end{align*}
by definition of $\Phi_{i,j}^{(k)}$. Since $X^{(1)}, \ldots, X^{(k-1)} \in [0,1]$, $X^{(k)} = t \in [0,1]$ and $X^{(k+1)}, \ldots, X^{(d_x n)} \in [0,1] \setminus \Omega([0,1], K, \delta)$ are arbitrary, we obtain
\begin{align*}
|\Phi_{i,j}^{(k)}(\bm{X}) - F_{i,j}(\bm{X})| \leq \varepsilon + k \cdot K_{\mathcal{H}} \delta^{\gamma} \ \text{ for any } \bm{X} \in \Omega_{k}.
\end{align*}
This completes the induction.

\textbf{Step 2:} Recall that $\bm{\Phi} = (\Phi_{i,j})_{i \in [d_x], j \in [n]}$ is defined by (\ref{eq: 12}). We now prove that if $\bm{G} \in \mathcal{T}_{d_x, d_x}(D, H, S, W, L)$, then
\begin{align*}
\bm{\Phi} \in \mathcal{T}_{d_x, d_x}(3^{d_x n} \max\{D, 5 d_x\}, 3^{d_x n} H, S, 3^{d_x n} \max\{W, 14 d_x\}, L + 2 d_x n).
\end{align*}

The observation here is that, to compute $\bm{\Phi} = \bm{\Phi}^{(d_x n)}$, we first evaluate 
\begin{align*}
\bm{\Phi}^{(d_x n-1)}(\cdot + c_{d_x n} \delta \bm{E}^{(d_x n)}) \ \text{ for each } c_{d_x n} \in \{-1,0,1\}.
\end{align*}
Each such evaluation, in turn, requires computing 
\begin{align*}
\bm{\Phi}^{(d_x n-2)}(\cdot + c_{d_x n-1} \delta \bm{E}^{(d_x n-1)} + c_{d_x n} \delta \bm{E}^{(d_x n)}) \ \text{ for each } c_{d_x n-1} \in \{-1,0,1\}.
\end{align*}
Continuing this process recursively, determining $\bm{\Phi}$ ultimately requires evaluating 
\begin{align*}
\textstyle \bm{\Phi}^{(0)}(\cdot + \sum_{l=1}^{d_x n} c_l \delta \bm{E}^{(l)}) \ \text{ for every } (c_1, \ldots, c_{d_x n}) \in \{-1,0,1\}^{d_x n}.
\end{align*}
Conversely, since $\bm{\Phi}^{(0)} = \bm{G}$ by definition, assuming that we have access to all functions $\bm{G}(\cdot + \sum_{l=1}^{d_x n} c_l \delta \bm{E}^{(l)})$, we can iteratively apply the $\operatorname{mid}$ function to $\bm{\Phi}^{(k)}$ to recover $\bm{\Phi}^{(k+1)}$, following the same construction as in (\ref{eq: 12}), and ultimately get $\bm{\Phi}$. On the other hand, each function $\bm{G}(\cdot + \sum_{k=1}^{d_x n} c_k \delta \bm{E}^{(k)})$ is a Transformer network thanks to positional encoding, and the $\operatorname{mid}$ function can be implemented by feed-forward layers and vectorized operations, thereby completing the construction. The details are given below.

Fixing $k \in \{0, 1, \ldots, d_x n-1\}$, we reindex the functions  
\begin{align*}
\left\{\bm{\Phi}^{(k)}(\cdot + \textstyle\sum_{l=k+1}^{d_x n} c_l \delta \bm{E}^{(l)}): (c_{k+1}, \ldots, c_{d_x n}) \in \{-1,0,1\}^{d_x n - k}\right\}
\end{align*}
as $\{\bm{\Phi}^{(k)}_l\}_{l=1}^{3^{d_x n - k}}$ (set $\bm{\Phi}^{(d_x n)}_1 = \bm{\Phi}^{(d_x n)}$ for notational convenience), such that
\begin{align*}
\left[\bm{\Phi}^{(k+1)}_l\right]_{i,j} = \operatorname{mid}\left(\left[\bm{\Phi}^{(k)}_{3l-2}\right]_{i,j}, \left[\bm{\Phi}^{(k)}_{3l-1}\right]_{i,j}, \left[\bm{\Phi}^{(k)}_{3l}\right]_{i,j}\right)
\end{align*}
for all $i \in [d_x], j \in [n]$, which aligns with (\ref{eq: 12}). Since $\operatorname{mid}(\cdot, \cdot, \cdot) \in \mathcal{FNN}_{3,1}(14, 2)$ by Lemma \ref{lemma: 6}, there exists an FNN $\widetilde{\mathcal{N}} \in \mathcal{FNN}_{3 d_x, d_x}(14 d_x, 2)$, such that for all $j \in [n]$,
\begin{align*}
\widetilde{\mathcal{N}}\begin{pmatrix}
\left[\bm{\Phi}^{(k)}_{3l-2}\right]_{:,j} \\
\left[\bm{\Phi}^{(k)}_{3l-1}\right]_{:,j} \\
\left[\bm{\Phi}^{(k)}_{3l}\right]_{:,j}
\end{pmatrix} = \left[\bm{\Phi}^{(k+1)}_l\right]_{:,j}.
\end{align*}
We then concatenate $\widetilde{\mathcal{N}}$ in parallel to construct a new FNN 
\begin{align*}
\widetilde{\mathcal{N}}^{(k)} \in \mathcal{FNN}_{3 d_x \cdot 3^{d_x n - k - 1}, d_x \cdot 3^{d_x n - k - 1}} (14 d_x \cdot 3^{d_x n - k - 1}, 2)
\end{align*}
such that
\begin{align*}
\widetilde{\mathcal{N}}^{(k)} \begin{pmatrix}
\left[\bm{\Phi}^{(k)}_{1}\right]_{:,j} \\
\left[\bm{\Phi}^{(k)}_{2}\right]_{:,j} \\
\left[\bm{\Phi}^{(k)}_{3}\right]_{:,j} \\
\vdots \\
\left[\bm{\Phi}^{(k)}_{3^{d_x n - k} - 2}\right]_{:,j} \\
\left[\bm{\Phi}^{(k)}_{3^{d_x n - k} - 1}\right]_{:,j} \\
\left[\bm{\Phi}^{(k)}_{3^{d_x n - k}}\right]_{:,j} 
\end{pmatrix} = \begin{pmatrix}
\left[\bm{\Phi}^{(k+1)}_{1}\right]_{:,j} \\
\vdots \\
\left[\bm{\Phi}^{(k+1)}_{3^{d_x n - k - 1}}\right]_{:,j} \\
\end{pmatrix}.
\end{align*}
By recursively composing $\widetilde{\mathcal{N}}^{(k)}$ for each $k \in \{0, 1, \ldots, d_x n-1\}$, we obtain 
\begin{align*}
\widetilde{\mathcal{N}}^{(d_x n-1)} \circ \widetilde{\mathcal{N}}^{(d_x n-2)} \circ \cdots \circ \widetilde{\mathcal{N}}^{(0)} \in \mathcal{FNN}_{d_x 3^{d_x n}, d_x} (14 d_x 3^{d_x n - 1}, 2 d_x n),
\end{align*}
which, by construction, satisfies
\begin{align*}
& \widetilde{\mathcal{N}}^{(d_x n-1)} \circ \widetilde{\mathcal{N}}^{(d_x n-2)} \circ \cdots \circ \widetilde{\mathcal{N}}^{(0)} \begin{pmatrix}
\left[\bm{\Phi}^{(0)}_{1}\right]_{:,j} \\
\left[\bm{\Phi}^{(0)}_{2}\right]_{:,j} \\
\left[\bm{\Phi}^{(0)}_{3}\right]_{:,j} \\
\vdots \\
\left[\bm{\Phi}^{(0)}_{3^{d_x n} - 2}\right]_{:,j} \\
\left[\bm{\Phi}^{(0)}_{3^{d_x n} - 1}\right]_{:,j} \\
\left[\bm{\Phi}^{(0)}_{3^{d_x n}}\right]_{:,j} \\
\end{pmatrix} \\
&= \widetilde{\mathcal{N}}^{(d_x n-1)} \circ \widetilde{\mathcal{N}}^{(d_x n-2)} \circ \cdots \circ \widetilde{\mathcal{N}}^{(1)} \begin{pmatrix}
\left[\bm{\Phi}^{(1)}_{1}\right]_{:,j} \\
\vdots \\
\left[\bm{\Phi}^{(1)}_{3^{d_x n - 1}}\right]_{:,j} \\
\end{pmatrix} \\
& ~\vdots \\
&= \left[\bm{\Phi}^{(d_x n)}\right]_{:,j},
\end{align*}
for each $j \in [n]$. Furthermore, Lemma \ref{lemma: 10} guarantees that any token-wise FNN can be expressed in terms of feed-forward layers. We have (set $W = 14 d_x 3^{d_x n - 1}$ and $L = 2 d_x n$ therein) an embedding map $\mathcal{E}_{in}: \mathbb{R}^{d_x 3^{d_x n} \times n} \rightarrow \mathbb{R}^{14 d_x 3^{d_x n - 1} \times n}$, a projection map $\mathcal{E}_{out}: \mathbb{R}^{14 d_x 3^{d_x n - 1} \times n} \rightarrow \mathbb{R}^{d_x \times n}$, and $2 d_x n$ feed-forward layers $\mathcal{F}_{L+2 d_x n}^{(FF)}, \ldots, \mathcal{F}_{L+1}^{(FF)}$, each with width at most $3 \cdot 14 d_x 3^{d_x n - 1} = 14 d_x 3^{d_x n}$, such that
\begin{align}\label{eq: 13}
\mathcal{E}_{out} \circ \mathcal{F}_{L+2 d_x n}^{(FF)} \circ \cdots \circ \mathcal{F}_{L+1}^{(FF)} \circ \mathcal{E}_{in} \begin{pmatrix}
\bm{\Phi}^{(0)}_{1} \\
\bm{\Phi}^{(0)}_{2} \\
\vdots \\
\bm{\Phi}^{(0)}_{3^{d_x n}} \\
\end{pmatrix} = \bm{\Phi}^{(d_x n)}.
\end{align}

Due to the positional encoding, each function
\begin{align*}
\textstyle \bm{G}(\cdot + \sum_{l=1}^{d_x n} c_k \delta \bm{E}^{(l)}) \in \mathcal{T}_{d_x, d_x}(D, H, S, W, L).
\end{align*}
Recall that $\bm{G} = \bm{\Phi}^{(0)}$ and $\{\bm{\Phi}^{(0)}_l\}_{l=1}^{3^{d_x n}}$ is a reordering of $\{\bm{\Phi}^{(0)}(\cdot + \sum_{l=1}^{d_x n} c_l \delta \bm{E}^{(l)})\}$. By concatenation of Transformers (see Proposition \ref{pro: 1}), there exists a Transformer network
\begin{align*}
\mathcal{N} \in \mathcal{T}_{d_x, d_x 3^{d_x n}}(3^{d_x n} D, 3^{d_x n} H, S, 3^{d_x n} W, L)
\end{align*}
such that
\begin{align*}
\mathcal{N}(\bm{X}) = \begin{pmatrix}
\bm{\Phi}^{(0)}_{1} \\
\bm{\Phi}^{(0)}_{2} \\
\vdots \\
\bm{\Phi}^{(0)}_{3^{d_x n}} \\
\end{pmatrix}.
\end{align*}
Together with (\ref{eq: 13}) and $\bm{\Phi}^{(d_x n)} = \bm{\Phi}$, we have
\begin{align*}
& \mathcal{E}_{out} \circ \mathcal{F}_{L+2 d_x n}^{(FF)} \circ \cdots \circ \mathcal{F}_{L+1}^{(FF)} \circ \mathcal{E}_{in} \circ \mathcal{N} (\bm{X}) \\
& = \mathcal{E}_{out} \circ \mathcal{F}_{L+2 d_x n}^{(FF)} \circ \cdots \circ \mathcal{F}_{L+1}^{(FF)} \circ \mathcal{E}_{in} \begin{pmatrix}
\bm{\Phi}^{(0)}_{1} \\
\bm{\Phi}^{(0)}_{2} \\
\vdots \\
\bm{\Phi}^{(0)}_{3^{d_x n}} \\
\end{pmatrix} \\
& = \bm{\Phi}^{(d_x n)} = \bm{\Phi}(\bm{X}),
\end{align*}
thereby
\begin{align*}
\bm{\Phi} & \in \mathcal{T}_{d_x, d_x}(\max\{3^{d_x n} D, 14 d_x 3^{d_x n - 1}\}, 3^{d_x n} H, S, \max\{3^{d_x n} W, 14 d_x 3^{d_x n}\}, L + 2 d_x n) \\
& \subseteq \mathcal{T}_{d_x, d_x}(3^{d_x n} \max\{D, 5 d_x\}, 3^{d_x n} H, S, 3^{d_x n} \max\{W, 14 d_x\}, L + 2 d_x n),
\end{align*}
which completes the proof.
\end{proof}

\begin{lemma}[Lemma 3.1 of \cite{lu2021deep}]\label{lemma: 6}
The middle value function $\operatorname{mid} (x_1, x_2, x_3) \in \mathcal{FNN}_{3,1}(14, 2)$.
\end{lemma}

\begin{lemma}[Lemma 3.3 of \cite{lu2021deep}]\label{lemma: 5}
Given any $\varepsilon > 0$, $K \in \mathbb{N}$, and $\delta \in (0, \frac{1}{3 K}]$, assume that $f \in \mathcal{H}^{\gamma} ([0,1], K_{\mathcal{H}})$ and $g: \mathbb{R} \rightarrow \mathbb{R}$ is a general function with
\begin{align*}
|g(x)-f(x)| \leq \varepsilon, \text{ for any } x \in [0,1] \setminus \Omega([0,1], K, \delta).
\end{align*}
Then
\begin{align*}
|\phi(x)-f(x)| \leq \varepsilon + K_{\mathcal{H}} \delta^{\gamma} \ \text{ for any } x \in [0,1],
\end{align*}
where
\begin{align*}
\phi(x) := \operatorname{mid}\left(g(x-\delta), g(x), g(x+\delta)\right) \ \text{ for any } x \in \mathbb{R}.
\end{align*}
\end{lemma}

\begin{proof}[Proof of Theorem \ref{thm: 4}]

\textbf{Case 1:} $p \in [1,\infty)$. Let 
\begin{align*}
\mathcal{N} \in \mathcal{T}_{d_x, d_x}(d_x, 1, 1, 5 n K^{d_x n}, 2)
\end{align*}
be as in Proposition \ref{pro: 2}. By Proposition \ref{pro: 2} and noting that the Lebesgue measure of $\Omega([0,1]^{d_x \times n}, K, \delta)$ is at most $d_x n K \delta$, we have
\begin{align*}
& \|\mathcal{N} - \bm{F}\|_{L^p([0,1]^{d_x \times n})}^p \\
&= \int_{[0,1]^{d_x \times n}} \|\mathcal{N}(\bm{X}) - \bm{F}(\bm{X})\|_F^p \dd \bm{X} \\
&= \int_{\Omega([0,1]^{d_x \times n}, K, \delta)} \|\mathcal{N}(\bm{X}) - \bm{F}(\bm{X})\|_F^p \dd \bm{X} + \int_{[0,1]^{d_x \times n} \setminus \Omega([0,1]^{d_x \times n}, K, \delta)} \|\mathcal{N}(\bm{X}) - \bm{F}(\bm{X})\|_F^p \dd \bm{X} \\
&\leq \int_{\Omega([0,1]^{d_x \times n}, K, \delta)} \|\mathcal{N}(\bm{X}) - \bm{F}(\bm{X})\|_F^p \dd \bm{X} \\
&~~~ + \int_{[0,1]^{d_x \times n} \setminus \Omega([0,1]^{d_x \times n}, K, \delta)} (d_x n)^{\max\{0, \frac{p}{2}-1\}} \sum_{i=1}^{d_x} \sum_{j=1}^n |\mathcal{N}_{i,j}(\bm{X}) - F_{i,j}(\bm{X})|^p \dd \bm{X} \\
&\leq (2 \sqrt{d_x n} K_{\mathcal{H}})^p \cdot d_x n K \delta + (d_x n)^{1 + \max\{0, \frac{p}{2}-1\}} (K_{\mathcal{H}} (d_x n)^{\gamma/2} K^{-\gamma})^p \\
&\leq 2^p (d_x n)^{2 p} K_{\mathcal{H}}^p ((K \delta)^{\frac{1}{p}} + K^{-\gamma})^p,
\end{align*}
using for the last inequality that $\gamma \in (0,1]$, $\max\{a, b\} \leq a + b$ for any $a, b \geq 0$, and $a^p + b^p \leq (a+b)^p$ for all $p \geq 1$ and $a, b \geq 0$. Hence, 
\begin{align*}
\|\mathcal{N} - \bm{F}\|_{L^p([0,1]^{d_x \times n})} \leq 2 (d_x n)^{2} K_{\mathcal{H}} ((K \delta)^{\frac{1}{p}} + K^{-\gamma}).
\end{align*}
Choosing $\delta \leq K^{-p \gamma - 1}$ and $K \geq \varepsilon^{-1/\gamma}$ so that $K^{d_x n} = \lceil\varepsilon^{-\frac{d_x n}{\gamma}}\rceil$, we conclude
\begin{align*}
\|\mathcal{N} - \bm{F}\|_{L^p([0,1]^{d_x \times n})} \leq 4 (d_x n)^2 K_{\mathcal{H}} \varepsilon
\end{align*}
and
\begin{align*}
\mathcal{N} \in \mathcal{T}_{d_x, d_x}(d_x, 1, 1, 5 n \lceil\varepsilon^{-\frac{d_x n}{\gamma}}\rceil, 2).
\end{align*}

\textbf{Case 2:} $p = \infty$. By Proposition \ref{pro: 2}, there exists a Transformer network
\begin{align*}
\widetilde{\mathcal{N}} \in \mathcal{T}_{d_x, d_x}(d_x, 1, 1, 5 n K^{d_x n}, 2)
\end{align*}
such that
\begin{align*}
|\widetilde{\mathcal{N}}_{i,j}(\bm{X}) - F_{i,j}(\bm{X})| \leq (d_x n)^{\gamma/2} K_{\mathcal{H}} K^{-\gamma}
\end{align*}
for any $i \in [d_x]$, $j \in [n]$ and $\bm{X} \in [0,1]^{d_x \times n} \setminus \Omega([0,1]^{d_x \times n}, K, \delta)$. By Proposition \ref{pro: 3} (assume that $5 n K^{d_x n} \geq 14 d_x$), there exists a new Transformer network
\begin{align*}
\mathcal{N} \in \mathcal{T}_{d_x, d_x}(5 d_x 3^{d_x n}, 3^{d_x n}, 1, 5 n 3^{d_x n} K^{d_x n}, 2 + 2 d_x n),
\end{align*}
such that
\begin{align*}
|\mathcal{N}_{i,j}(\bm{X}) - F_{i,j}(\bm{X})| \leq (d_x n)^{\gamma/2} K_{\mathcal{H}} K^{-\gamma} + d_x n K_{\mathcal{H}} \delta^{\gamma}
\end{align*}
for any $i \in [d_x]$, $j \in [n]$ and $\bm{X} \in [0,1]^{d_x \times n}$. This implies
\begin{align*}
\|\mathcal{N} - \bm{F}\|_{L^{\infty}([0,1]^{d_x \times n})} &= \sup_{\bm{X} \in [0,1]^{d_x \times n}} \|\mathcal{N}(\bm{X}) - \bm{F}(\bm{X})\|_F \\
&\leq \sup_{\bm{X} \in [0,1]^{d_x \times n}} \sum_{i=1}^{d_x} \sum_{j=1}^n |\mathcal{N}_{i,j}(\bm{X}) - F_{i,j}(\bm{X})| \\
&\leq \sum_{i=1}^{d_x} \sum_{j=1}^n \sup_{\bm{X} \in [0,1]^{d_x \times n}} |\mathcal{N}_{i,j}(\bm{X}) - F_{i,j}(\bm{X})| \\
&\leq (d_x n)^{1 + \gamma/2} K_{\mathcal{H}} K^{-\gamma} + (d_x n)^2 K_{\mathcal{H}} \delta^{\gamma}.
\end{align*}
Choosing $\delta \in (0,\frac{1}{3K}]$ sufficiently small and $K \geq \varepsilon^{-1/\gamma}$ so that $K^{d_x n} = \lceil\varepsilon^{-\frac{d_x n}{\gamma}}\rceil$, we conclude
\begin{align*}
\|\mathcal{N} - \bm{F}\|_{L^{\infty}([0,1]^{d_x \times n})} \leq 4 (d_x n)^2 K_{\mathcal{H}} \varepsilon
\end{align*}
and
\begin{align*}
\mathcal{N} \in \mathcal{T}_{d_x, d_x}(5 d_x 3^{d_x n}, 3^{d_x n}, 1, 5 n 3^{d_x n} \lceil\varepsilon^{-\frac{d_x n}{\gamma}}\rceil, 2 + 2 d_x n).
\end{align*}
This completes the proof.
\end{proof}

\begin{proof}[Proof of Theorem \ref{thm: 5}]

Let $K$, $\mathbb{G}_K$ and $\omega_{\bm{G}}$ be as defined in the proof of Proposition \ref{pro: 2}. We approximate the target function $\bm{F}$ using a piecewise constant function, where the value in each cell is given by the average of $\bm{F}$ over that cell. Define
\begin{align*}
\overline{\bm{F}}(\bm{X}) = \sum_{\bm{G} \in \mathbb{G}_K} \bm{F}_{\bm{G}} \mathbbm{1}_{\omega_{\bm{G}}}(\bm{X}),
\end{align*}
where
\begin{align*}
[F_{\bm{G}}]_{i,j} = K^{d_x n} \int_{\omega_{\bm{G}}} F_{i,j}(\bm{X}) \dd \bm{X}, \quad i \in [d_y], j \in [n].
\end{align*}
Since each cell $\omega_{\bm{G}}$ is a bounded convex domain, Poincar\'e inequality gives, for any $p \in [1 ,\infty]$,
\begin{align*}
\|[F_{\bm{G}}]_{i,j} - F_{i,j}\|_{L^p(\omega_{\bm{G}})} \leq C \|\nabla F_{i,j}\|_{L^p(\omega_{\bm{G}})} K^{-1},
\end{align*}
where $C$ is a constant depending only on $d_x n$, and $\|\nabla F\|_{L^p(\omega)}$ denotes the $L^p$-norm of the Frobenius norm of $\nabla F$ (see \cite{evans2010partial, dekel2004bramble}). Summing over all grid cells and using that $F_{i,j} \in \mathcal{W}^{1,p}([0,1]^{d_x \times n}, K_\mathcal{W})$ implies $\|\nabla F_{i,j}\|_{L^p([0,1]^{d_x \times n})} \leq (d_x n)^{\max\{0,\frac{1}{2}-\frac{1}{p}\}} K_\mathcal{W}$, we obtain
\begin{align}\label{eq: 25}
\begin{aligned}
\|F_{i,j} - \overline{F}_{i,j}\|_{L^p([0,1]^{d_x \times n})} &= \begin{cases}
\left(\sum_{\bm{G} \in \mathbb{G}_K} \|F_{i,j} - [F_{\bm{G}}]_{i,j}\|_{L^p(\omega_{\bm{G}})}^p\right)^{1/p} & \text{ if } p<\infty \\ 
\sup_{\bm{G} \in \mathbb{G}_K} \|F_{i,j} - [F_{\bm{G}}]_{i,j}\|_{L^{\infty}(\omega_{\bm{G}})} & \text{ if } p=\infty
\end{cases} \\
&\leq C \|\nabla F_{i,j}\|_{L^p([0,1]^{d_x \times n})} K^{-1}\\
&\leq C (d_x n)^{\max\{0,\frac{1}{2}-\frac{1}{p}\}} K_\mathcal{W} K^{-1},
\end{aligned}
\end{align}
for any $p \in [1,\infty]$.

From \textbf{Step 2} and \textbf{Step 3} of Proposition \ref{pro: 2}, there exists a Transformer network
\begin{align*}
\mathcal{N} \in \mathcal{T}_{d_x, d_x}(d_x, 1, 1, 5 n K^{d_x n}, 2)
\end{align*}
such that
\begin{align*}
\mathcal{N}(\bm{X}) = \overline{\bm{F}}(\bm{X}) \ \text{ for any } \bm{X} \in [0,1]^{d_x \times n} \setminus \Omega([0,1]^{d_x \times n}, K, \delta)
\end{align*}
and 
\begin{align*}
\|\mathcal{N}(\bm{X})\|_F \leq \sqrt{d_x n} K_{\mathcal{W}} \ \text{ for any } \bm{X} \in \mathbb{R}^{d_x \times n}.
\end{align*}

Since the Lebesgue measure of $\Omega([0,1]^{d_x \times n}, K, \delta)$ is at most $d_x n K \delta$, for $p \in [1,\infty)$, we have 
\begin{align*}
& \|\mathcal{N} - \bm{F}\|_{L^p([0,1]^{d_x \times n})}^p \\
&= \int_{[0,1]^{d_x \times n}} \|\mathcal{N}(\bm{X}) - \bm{F}(\bm{X})\|_F^p \dd \bm{X} \\
&= \int_{\Omega([0,1]^{d_x \times n}, K, \delta)} \|\mathcal{N}(\bm{X}) - \bm{F}(\bm{X})\|_F^p \dd \bm{X} + \int_{[0,1]^{d_x \times n} \setminus \Omega([0,1]^{d_x \times n}, K, \delta)} \|\overline{\bm{F}}(\bm{X}) - \bm{F}(\bm{X})\|_F^p \dd \bm{X} \\
&\leq \int_{\Omega([0,1]^{d_x \times n}, K, \delta)} \|\mathcal{N}(\bm{X}) - \bm{F}(\bm{X})\|_F^p \dd \bm{X} \\
&~~~ + \int_{[0,1]^{d_x \times n} \setminus \Omega([0,1]^{d_x \times n}, K, \delta)} (d_x n)^{\max\{0, \frac{p}{2}-1\}} \sum_{i=1}^{d_x} \sum_{j=1}^n |\overline{F}_{i,j}(\bm{X}) - F_{i,j}(\bm{X})|^p \dd \bm{X} \\
&\leq (2 \sqrt{d_x n} K_{\mathcal{W}})^p \cdot d_x n K \delta + (d_x n)^{\max\{0, \frac{p}{2}-1\}} \sum_{i=1}^{d_x} \sum_{j=1}^n \left(C (d_x n)^{\max\{0,\frac{1}{2}-\frac{1}{p}\}} K_\mathcal{W} K^{-1}\right)^p \\
&\leq (2C)^p (d_x n)^{2p} K_{\mathcal{W}}^p ((K \delta)^{\frac{1}{p}} + K^{-1})^p,
\end{align*}
using for the last inequality that $p \geq 1$, $\max\{a, b\} \leq a + b$ for any $a, b \geq 0$, and $a^p + b^p \leq (a+b)^p$ for all $p \geq 1$ and $a, b \geq 0$. Hence,
\begin{align*}
\|\mathcal{N} - \bm{F}\|_{L^p([0,1]^{d_x \times n})} \leq 2C (d_x n)^2 K_{\mathcal{W}} ((K \delta)^{\frac{1}{p}} + K^{-1}).
\end{align*}
Choosing $\delta \leq K^{-p - 1}$ and $K \geq \varepsilon^{-1}$ so that $K^{d_x n} = \lceil\varepsilon^{-d_x n}\rceil$, we conclude
\begin{align*}
\|\mathcal{N} - \bm{F}\|_{L^p([0,1]^{d_x \times n})} \leq 4C (d_x n)^2 K_{\mathcal{W}} \varepsilon
\end{align*}
and
\begin{align*}
\mathcal{N} \in \mathcal{T}_{d_x, d_x}(d_x, 1, 1, 5 n \lceil\varepsilon^{-d_x n}\rceil, 2).
\end{align*}
This completes the proof.
\end{proof}

\begin{lemma}
\label{lemma: 1}
Let $d_x, d_y, r \in \mathbb{N}$ with $d_x \geq d_y$. Let $\{(\bm{x}_i, \bm{y}_i)\}_{i=1}^r$ be a set of input-output pairs such that $\bm{x}_i \in \mathbb{R}^{d_x}, \bm{y}_i \in \mathbb{R}^{d_y}, i \in [r]$ and $\bm{x}_i \neq \bm{x}_j$ if $i \neq j$. Then, there exists a feed-forward layer $\mathcal{F}^{(FF)}: \mathbb{R}^{d_x} \rightarrow \mathbb{R}^{d_x}$ with width at most $3r + 2d_x$ such that
\begin{align*}
\mathcal{F}^{(FF)}(\bm{x}_i) = \begin{pmatrix}
\bm{y}_i \\
\bm{0}
\end{pmatrix} \ \text{ for all } i \in [r],
\end{align*}
and 
$\|\mathcal{F}^{(FF)}(\bm{z})\| \leq \max_i \|\bm{y}_i\|$ for any $\bm{z} \in \mathbb{R}^{d_x}$.
\end{lemma}

\begin{proof}
Let $R>0$ be determined later. Since $\bm{x}_i, i \in [r]$ are pairwise distinct, we can find $\bm{v} \in \mathbb{R}^{d_x}$ such that $\bm{v}^\top \bm{x}_i, i \in [r]$ are distinct. The existence of $\bm{v}$ can be found in \cite[Lemma 13]{park2021provable}. We define
\begin{align*}
\bm{A}_i^{(1)}=R \bm{1}_3 \bm{v}^\top, \quad
\bm{b}_i^{(1)}=\begin{pmatrix}
-R\bm{v}^\top\bm{x}_i-1\\
-R\bm{v}^\top\bm{x}_i\\
-R\bm{v}^\top\bm{x}_i+1
\end{pmatrix}, \quad
\bm{A}_i^{(2)}=\begin{pmatrix}
\bm{y}_i \\
\bm{0}
\end{pmatrix}
\begin{pmatrix}
1, -2, 1
\end{pmatrix},\quad
\bm{b}_i^{(2)}=\bm{0}.
\end{align*}
Then, by direct calculation, we obtain
\begin{align*}
& \bm{A}_i^{(2)} \sigma_R [\bm{A}_i^{(1)} \bm{x} + \bm{b}_i^{(1)}] + \bm{b}_i^{(2)} \\
= & \begin{pmatrix}
\bm{y}_i \\
\bm{0}
\end{pmatrix} 
\left(\sigma_R [R\bm{v}^\top (\bm{x}-\bm{x}_i)-1] - 2\sigma_R [R\bm{v}^\top (\bm{x}-\bm{x}_i)] + \sigma_R [R\bm{v}^\top (\bm{x}-\bm{x}_i)+1]\right) \\
= & \begin{pmatrix}
\bm{y}_i \\
\bm{0}
\end{pmatrix} I_i (\bm{x}),
\end{align*}
where $I_i (\bm{x})$ is the hat function with $I_i (\bm{x}_i) = 1$ and $I_i (\bm{x}) = 0$ if $|\bm{v}^\top (\bm{x}-\bm{x}_i)| \geq 1/R$. To ensure the supports of $I_i (\bm{x})$ for all $i \in [r]$ are disjoint, we choose $R > 2/\min_{i \neq j} |\bm{v}^\top (\bm{x}_i-\bm{x}_j)|$. Define
\begin{align*}
\bm{A}^{(1)} = \begin{pmatrix}
\bm{A}_1^{(1)} \\
\vdots \\
\bm{A}_r^{(1)} \\
\bm{I}_{d_x} \\
-\bm{I}_{d_x}
\end{pmatrix}, \quad \bm{b}^{(1)} = 
\begin{pmatrix}
\bm{b}_1^{(1)} \\
\vdots \\
\bm{b}_r^{(1)} \\
\bm{0} \\
\bm{0}
\end{pmatrix}, \quad \bm{A}^{(2)} = \begin{pmatrix}
\bm{A}_1^{(2)}, \ldots, \bm{A}_r^{(2)}, -\bm{I}_{d_x}, \bm{I}_{d_x}
\end{pmatrix}, \quad \bm{b}^{(2)} = \bm{0},
\end{align*}
and let
\begin{align*}
\mathcal{F}^{(FF)}(\bm{x}) &= \bm{x} + \bm{A}^{(2)} \sigma_R [\bm{A}^{(1)} \bm{x} + \bm{b}^{(1)}] + \bm{b}^{(2)} \\
& = \sum_{i=1}^r \begin{pmatrix}
\bm{y}_i \\
\bm{0}
\end{pmatrix} I_i (\bm{x}).
\end{align*}
We complete the proof by verifying that
\begin{align*}
\mathcal{F}^{(FF)}(\bm{x}_k) = \sum_{i=1}^r \begin{pmatrix}
\bm{y}_i \\
\bm{0}
\end{pmatrix} I_i (\bm{x}_k) = 
\begin{pmatrix}
\bm{y}_k \\
\bm{0}
\end{pmatrix}
\end{align*}
and
\begin{align*}
\|\mathcal{F}^{(FF)}(\bm{x})\| &= \left\|\sum_{i=1}^r \begin{pmatrix}
\bm{y}_i \\
\bm{0}
\end{pmatrix} I_i (\bm{x})\right\| \\
& \leq \max_i \|\bm{y}_i\| \left\|\sum_{i=1}^r I_i (\bm{x})\right\| \\
& \leq \max_i \|\bm{y}_i\|.
\end{align*}

\end{proof}

\subsection{Proof of Theorem \ref{thm: 7}}\label{sec: 2}

We introduce sample complexities, which measure the richness of the function class in different aspects, and use them to bound the generalization error.

\begin{definition}[VC-dimension]
Let $\mathcal{H}$ be a class of real-valued functions defined on $\Omega$. The VC-dimension of $\mathcal{H}$, denoted by $\operatorname{VCDim}(\mathcal{H})$, is the largest integer $N$ for which there exist points $x_1, \ldots, x_N \in \Omega$ such that
\begin{align*}
|\{\operatorname{sgn}(h(x_1)), \ldots, \operatorname{sgn}(h(x_N)): h \in \mathcal{H}\}| = 2^N.
\end{align*}
\end{definition}

\begin{definition}[Pseudo-dimension]
Let $\mathcal{H}$ be a class of real-valued functions defined on $\Omega$. The pseudo-dimension of $\mathcal{H}$, denoted by $\operatorname{Pdim}(\mathcal{H})$, is the largest integer $N$ for which there exist points $x_1, \ldots, x_N \in \Omega$ and constants $c_1, \ldots, c_N \in \mathbb{R}$ such that
\begin{align*}
|\{\operatorname{sgn}(h(x_1)-c_1), \ldots, \operatorname{sgn}(h(x_N)-c_N): h \in \mathcal{H}\}| = 2^N.
\end{align*}
\end{definition}

\begin{definition}[Covering number]
Let $\rho$ be a pseudo-metric on $\mathcal{M}$ and $S \subseteq \mathcal{M}$. For any $\delta>0$, a set $A \subseteq \mathcal{M}$ is called a $\delta$-covering of $S$ if for any $x \in S$ there exists $y \in A$ such that $\rho(x, y) \leq \delta$. The $\delta$-covering number of $S$, denoted by $\mathcal{N}(\delta, S, \rho)$, is the minimum cardinality of any $\delta$-covering of $S$.
\end{definition}

\begin{theorem}[Theorem 8.14 of \cite{anthony2009neural}]\label{thm: 6}
Let $h$ be a function from $\mathbb{R}^d \times \mathbb{R}^n$ to $\{0,1\}$, determining the class
\begin{align*}
\mathcal{H} = \{x \mapsto h(a, x): a \in \mathbb{R}^d\}.
\end{align*}
Suppose that $h$ can be computed by an algorithm that takes as input the pair $(a, x) \in \mathbb{R}^d \times \mathbb{R}^n$ and returns $h(a, x)$ after no more than $t$ of the following operations:
\begin{itemize}[itemsep=0em, labelwidth=1em, leftmargin=!]
\item the exponential function $\alpha \mapsto e^\alpha$ on real numbers,
\item the arithmetic operations $+$, $-$, $\times$, and $/$ on real numbers,
\item jumps conditioned on $>$, $\geq$, $<$, $\leq$, $=$, and $\neq$ comparisons of real numbers, and
\item output $0$ or $1$.
\end{itemize}
Then $\operatorname{VCdim}(\mathcal{H}) \leq t^2 d \left(d + 19 \log_2(9 d)\right)$. Furthermore, if the $t$ steps include no more than $q$ in which the exponential function is evaluated, then
\begin{align*}
\operatorname{VCdim}(\mathcal{H}) \leq (d(q+1))^2 + 11 d (q+1) \left(t + \log_2 (9 d (q+1))\right).
\end{align*}
\end{theorem}

Theorem \ref{thm: 6} gives bounds on the VC-dimension of a function class in terms of the number of arithmetic operations required to compute the functions. This result immediately implies a bound on the VC-dimension (or pseudo-dimension) for Transformer networks. By applying standard techniques in learning theory, one can further derive upper bounds for the covering number. The following lemma summarizes these bounds.

\begin{lemma}\label{lemma: 11}
Recall that $\mathcal{F} = \{f = \langle \mathcal{N}, \bm{E}\rangle: \mathcal{N} \in \mathcal{T}_{d_x, d_x}(D, H, S, W, L)\}$. Then the following bounds hold:
\begin{itemize}[itemsep=0em, labelwidth=1em, leftmargin=!]
\item $\operatorname{VCdim}(\mathcal{F}) \lesssim (H S + W)^2 D^2 H^2 L^4$,
\item $\operatorname{Pdim}(\mathcal{F}) \lesssim (H S + W)^2 D^2 H^2 L^4$,
\item $\sup_{\mathcal{X}} \log\mathcal{N} (\delta, \mathcal{C}_K \mathcal{F}, d_{\mathcal{X}, \infty}) \lesssim (H S + W)^2 D^2 H^2 L^4 \log\frac{m K}{\delta}$, where $\mathcal{X} = \{\bm{X}_i\}_{i=1}^m$ and 
\begin{align*}
d_{\mathcal{X}, \infty}(f,g) = \max_{i \in [m]} |f(\bm{X}_i)-g(\bm{X}_i)|.
\end{align*}
\end{itemize}
We hide constants that depend on $d_x$ and $n$.
\end{lemma}

\begin{proof}
Recall that $\mathcal{F} = \{f = \langle \mathcal{N}, \bm{E}\rangle: \mathcal{N} \in \mathcal{T}_{d_x, d_x}(D, H, S, W, L)\}$. By carefully counting the computational steps required to evaluate any $f \in \mathcal{F}$, we deduce that
\begin{itemize}[itemsep=0em, labelwidth=1em, leftmargin=!]
\item the total number of parameters is bounded by $d \lesssim (HS + W) DL$,
\item the total number of computational operations is bounded by $t \lesssim L (H D S n + H S n^2 + W D n)$,
\item the number of evaluations of the exponential function is bounded by $q \lesssim L H n^2$.
\end{itemize}
Theorem \ref{thm: 6} immediately implies that
\begin{align*}
\operatorname{VCdim}(\mathcal{F}) &\leq (d(q+1))^2 + 11 d (q+1) \left(t + \log_2 (9 d (q+1))\right) \\
&\lesssim (H S + W)^2 D^2 H^2 L^4,
\end{align*}
where we use $\log(x) \leq x$ for $x \geq 1$ and suppress constants that depend on $d_x$ and $n$.

For the pseudo-dimension, note that by definition $\operatorname{VCdim}(\{f(x) - r: f \in \mathcal{F}, r \in \mathbb{R}\}) = \operatorname{Pdim}(\mathcal{F})$. Using the same reasoning as above, we have
\begin{align*}
\operatorname{Pdim}(\mathcal{F}) \lesssim (H S + W)^2 D^2 H^2 L^4,
\end{align*}
again by Theorem \ref{thm: 6}.

Finally, by Theorem 12.2 of \cite{anthony2009neural}, we have
\begin{align*}
\log \mathcal{N} (\delta, \mathcal{C}_K \mathcal{F}, d_{\mathcal{X}, \infty}) &\leq \operatorname{Pdim}(\mathcal{C}_K \mathcal{F}) \log\frac{e m K}{\delta} \\
&\leq \operatorname{Pdim}(\mathcal{F}) \log\frac{e m K}{\delta} \\
&\lesssim (H S + W)^2 D^2 H^2 L^4 \log\frac{m K}{\delta}.
\end{align*}
Taking the supremum over all possible sample sets $\mathcal{X}$ completes the proof.
\end{proof}

\begin{proof}[Proof of Theorem \ref{thm: 7}]
Let $\mathcal{X}$ and $d_{\mathcal{X}, \infty}$ be defined as in Lemma \ref{lemma: 11}. Similar to the proof of \cite[Theorem 5]{jiao2024approximation}, given a random sample $\mathcal{D}_m = \{(\bm{x}_i, y_i)\}_{i=1}^m$, the excess risk can be decomposed as
\begin{align*}
\mathbb{E}_{\mathcal{D}_m} [\mathcal{R}(\mathcal{C}_{B_m} \hat{f}_m) - \mathcal{R}(f^*)] \lesssim \mathcal{E}_{app} + \mathcal{E}_{gen} + \mathcal{E}_{den},
\end{align*}
where 
\begin{align*}
\mathcal{E}_{app} &:= \inf_{f \in \mathcal{F}} \mathbb{E} [(f - f^*)^2], \\
\mathcal{E}_{gen} &:= \frac{B_m^2 k_m}{m} \sup_{|\mathcal{X}|=m} \log \mathcal{N}(m^{-1}, \mathcal{C}_{B_m} \mathcal{F}, d_{\mathcal{X}, \infty}), \\
\mathcal{E}_{den} &:= \frac{B_m^2 m}{k_m} \beta(k_m).
\end{align*}
Here, $k_m \in \mathbb{N}$ is a parameter to be chosen. It can be seen that the excess risk is bounded by the sum of the approximation error $\mathcal{E}_{app}$, the generalization error $\mathcal{E}_{gen}$, and the dependence error $\mathcal{E}_{den}$. Note that as $k_m$ increases, $\mathcal{E}_{den}$ decreases due to the monotonic decrease of the $\beta$-mixing coefficient $\beta(k_m)$, whereas $\mathcal{E}_{gen}$ increases. Besides, if we select a larger hypothesis class $\mathcal{F}$, then $\mathcal{E}_{app}$ decreases but $\mathcal{E}_{gen}$ increases because the covering number grows with the size of the hypothesis class. Therefore, to obtain a better convergence rate, we must carefully trade off these three errors by choosing an appropriate hypothesis class $\mathcal{F}$ and tuning the parameter $k_m$.

Since by assumption the density of $\Pi$ is upper bounded, Theorem \ref{thm: 4} implies that
\begin{align*}
\mathcal{E}_{app} \leq \inf_{f \in \mathcal{F}} \|f - f^*\|_{L^2([0,1]^{d_x \times n})}^2 \lesssim \varepsilon^2,
\end{align*}
where the hypothesis class
\begin{align*}
\mathcal{F} = \mathcal{F}(D_m \lesssim 1, H_m \lesssim 1, S_m \lesssim 1, W_m \lesssim \varepsilon^{-\frac{d_x n}{\gamma}}, L_m \lesssim 1).
\end{align*}
Then by Lemma \ref{lemma: 11},
\begin{align*}
\mathcal{E}_{gen} &\lesssim \frac{B_m^2 k_m}{m} (H S + W)^2 D^2 H^2 L^4 \log(m^2 B_m) \\
& \lesssim \frac{(\log m)^3 k_m}{m} \varepsilon^{-\frac{2 d_x n}{\gamma}},
\end{align*}
where we take $B_m \asymp \log m$.

We now consider three cases for the sequence $\{\bm{x}_i\}_{i=1}^m$.

\textbf{Case 1:} if $\{\bm{x}_i\}_{i=1}^m$ is geometrically $\beta$-mixing, i.e., $\beta(k_m) \leq \beta_0 \exp \left(-\beta_1 k_m^r\right)$ for some $r,\beta_0,\beta_1>0$, we set $k_m \asymp (\log m)^{1/r}$ so that $\beta(k_m) \lesssim 1/m^{100}$. Then,
\begin{align*}
\mathbb{E}_{\mathcal{D}_m} [\mathcal{R}(\mathcal{C}_{B_m} \hat{f}_m) - \mathcal{R}(f^*)] &\lesssim \varepsilon^2 + \frac{(\log m)^{3 + 1/r}}{m} \varepsilon^{-\frac{2 d_x n}{\gamma}} \\
&\lesssim m^{-\frac{\gamma}{\gamma + d_x n}} (\log m)^{3 + 1/r},
\end{align*}
where $\varepsilon$ is chosen as $\varepsilon \asymp m^{-\frac{\gamma}{2 \gamma + 2 d_x n}}$.

\textbf{Case 2:} if $\{\bm{x}_i\}_{i=1}^m$ is algebraically $\beta$-mixing, that is, $\beta(k_m) \leq \beta_0 / k_m^r$ for some $r,\beta_0>0$, then
\begin{align*}
\mathbb{E}_{\mathcal{D}_m} [\mathcal{R}(\mathcal{C}_{B_m} \hat{f}_m) - \mathcal{R}(f^*)] &\lesssim \varepsilon^2 + \frac{(\log m)^3 k_m}{m} \varepsilon^{-\frac{2 d_x n}{\gamma}} + \frac{(\log m)^2 m}{k_m^{r+1}} \\
&\lesssim m^{-\frac{r \gamma}{(r+2) \gamma + (r+1) d_x n}} (\log m)^3,
\end{align*}
where we use the AM-GM inequality and choose $\varepsilon \asymp m^{-\frac{r \gamma}{2(r+2) \gamma + 2(r+1) d_x n}}$ and $k_m \asymp m^{\frac{2 \gamma + d_x n}{(r+2) \gamma + (r+1) d_x n}}$.

\textbf{Case 3:} if $\{\bm{x}_i\}_{i=1}^m$ is a sequence of i.i.d. random variables, then $\beta(k_m) = 0$ for all $k_m \geq 1$. This implies
\begin{align*}
\mathbb{E}_{\mathcal{D}_m} [\mathcal{R}(\mathcal{C}_{B_m} \hat{f}_m) - \mathcal{R}(f^*)] &\lesssim \varepsilon^2 + \frac{(\log m)^3}{m} \varepsilon^{-\frac{2 d_x n}{\gamma}} \\
&\lesssim m^{-\frac{\gamma}{\gamma + d_x n}} (\log m)^3,
\end{align*}
where we choose $\varepsilon \asymp m^{-\frac{\gamma}{2 \gamma + 2 d_x n}}$. So we complete the proof.

\end{proof}

\subsection{Proof of Theorem \ref{thm: 8}}\label{sec: 3}

The original Kolmogorov-Arnold representation theorem states that for any continuous function $f: [0,1]^d \rightarrow \mathbb{R}$, there exist univariate continuous functions $g_q, \psi_{p, q}$ such that
\begin{align*}
f(x_1, \ldots, x_d) = \sum_{q=0}^{2 d} g_q \left(\sum_{p=1}^d \psi_{p, q} (x_p)\right).
\end{align*}
\cite{schmidt2021kolmogorov} derived modifications of this representation that transfer smoothness properties of the represented function to the outer function.

\begin{proposition}[Theorem 2 of \cite{schmidt2021kolmogorov}]\label{pro: 4}
For any fixed dimension $d \geq 2$, there exists a monotone function $\phi:[0,1] \rightarrow \mathcal{C}$ (the Cantor set) such that for any function $f \in \mathcal{H}^\gamma([0,1]^d, K_\mathcal{H})$ with some $\gamma \in (0,1]$, we can find a function $g \in \mathcal{H}^{\frac{\gamma \log 2}{d \log 3}}(\mathcal{C}, 2 \sqrt{d} K_\mathcal{H})$ such that
\begin{align}\label{eq: 18}
f(x_1, \ldots, x_d) = g\left(3 \sum_{p=1}^d 3^{-p} \phi(x_p)\right).
\end{align}
Moreover, for any $x \in [0,1]$ with its binary representation $x = [0.a_1^x a_2^x \ldots]_2$, the function $\phi$ is given explicitly by
\begin{align*}
\phi(x) = \sum_{j=1}^{\infty} \frac{2 a_j^x}{3^{1 + d(j-1)}} = [0 . (2 a_1^x) \underbrace{0 \ldots \ldots 0}_{(d-1) \text{-times}} (2 a_2^x) \underbrace{0 \ldots \ldots 0}_{(d-1) \text{-times}}]_3,
\end{align*}
where $[\cdot]_B$ denotes the $B$-adic expansion of a real number.
\end{proposition}

We note that a given real number can have multiple $B$-adic representations (for example, $[1]_{10} = [0.999\ldots]_{10}$), which may make $\phi$ not well-defined. To eliminate this ambiguity, we adopt the convention of using a unique $B$-adic representation for all real numbers. Observe that the argument of $g$ in (\ref{eq: 18}) satisfies
\begin{align}\label{eq: 23}
3 \sum_{p=1}^d 3^{-p} \phi(x_p) = [0 . (2 a_1^{x_1}) (2 a_1^{x_2}) \ldots (2 a_1^{x_d}) (2 a_2^{x_1}) \ldots]_3.
\end{align}
By construction, the Cantor set consists precisely of those numbers in $[0,1]$ whose ternary expansion contains only the digits $0$ and $2$. This shows that the mapping $3 \sum_{p=1}^d 3^{-p} \phi(x_p)$ indeed defines a bijection between $[0,1]^d$ and the Cantor set $\mathcal{C}$. Additionally, an approximation of $\phi$ with a truncation parameter $K$ is defined by
\begin{align}\label{eq: 21}
\phi_K(x) := \sum_{j=1}^K \frac{2 a_j^x}{3^{1+d(j-1)}},
\end{align}
which will be used in our construction.

\begin{proof}[Proof of Theorem \ref{thm: 8}]
By Proposition \ref{pro: 4} and the fact that $\phi_K$ approximates $\phi$, there exists a function $\bm{G}: \mathbb{R}^{d_x \times n} \rightarrow \mathbb{R}^{d_x \times n}$ with each entry $G_{u,v} \in \mathcal{H}^{\frac{\gamma \log 2}{d_x n \log 3}}(\mathcal{C}, 2 \sqrt{d_x n} K_\mathcal{H})$ such that
\begin{align}\label{eq: 22}
\begin{aligned}
\bm{F}(\bm{X}) &= \bm{G} \left(3 \sum_{p=1}^{d_x} \sum_{q=1}^{n} a_{p,q} \phi (X_{p,q})\right) \approx \bm{G} \left(3 \sum_{p=1}^{d_x} \sum_{q=1}^{n} a_{p,q} \phi_K (X_{p,q})\right).
\end{aligned}
\end{align}
We will construct a generalized Transformer network that approximates the latter mapping. In the proof below, for simplicity, we omit the placeholder zeros used for alignment.

\textbf{Step 1:} We first show that there exist $2K + 2$ generalized feed-forward layers $\mathcal{F}_1^{(GFF)}, \ldots, \mathcal{F}_{2K+2}^{(GFF)}$ such that 
\begin{align*}
\mathcal{F}_{2K+2}^{(GFF)} \circ \cdots \circ \mathcal{F}_1^{(GFF)}: \bm{X} \mapsto \bm{Z}_1,
\end{align*}
where
\begin{align*}
\bm{Z}_1 = 
\begin{pmatrix}
3 \sum_{p=1}^{d_x} a_{p,1} \widetilde{\phi}_K (X_{p,1}) & 3 \sum_{p=1}^{d_x} a_{p,2} \widetilde{\phi}_K (X_{p,2}) & \cdots & 3 \sum_{p=1}^{d_x} a_{p,n} \widetilde{\phi}_K (X_{p,n}) \\
\vdots & \vdots & & \vdots \\
3 \sum_{p=1}^{d_x} a_{p,1} \widetilde{\phi}_K (X_{p,1}) & 3 \sum_{p=1}^{d_x} a_{p,2} \widetilde{\phi}_K (X_{p,2}) & \cdots & 3 \sum_{p=1}^{d_x} a_{p,n} \widetilde{\phi}_K (X_{p,n})
\end{pmatrix}.
\end{align*}
Here, $a_{p,q} = \frac{1}{3^{(q-1)d_x+p}}$ and $\widetilde{\phi}_K$ is a function that satisfies
\begin{align*}
\widetilde{\phi}_K (x) = \begin{cases}
0, & \text{if } x < 0, \\
\phi_K(x), & \text{if } x \in \Omega_K \subseteq [0,1], \\
1, & \text{if } x > 1,
\end{cases}
\end{align*}
where $\phi_K$ is defined in (\ref{eq: 21}) and $\Omega_K \subseteq [0,1]$ has Lebesgue measure at least $1 - 2^{-K \gamma p}$. \cite[Theorem 3]{schmidt2021kolmogorov} guarantees the existence of an FNN $\widetilde{\phi}_K \in \mathcal{FNN}_{1,1}(4,2K)$ with the above properties.

By parallel computation, we can construct an FNN $\widetilde{\mathcal{N}}_1 \in \mathcal{FNN}_{1,1}(4n,2K)$ such that
\begin{align*}
\widetilde{\mathcal{N}}_1 (x) &= 
\left(1, 3^{-d_x}, 3^{-2 d_x}, \ldots, 3^{-(n-1) d_x}\right) \begin{pmatrix}
\widetilde{\phi}_K (x) \\
\widetilde{\phi}_K (x-2) \\
\widetilde{\phi}_K (x-4) \\
\vdots \\
\widetilde{\phi}_K (x-2(n-1))
\end{pmatrix} \\
&= \sum_{q=1}^n 3^{-(q-1) d_x} \widetilde{\phi}_K (x-2(q-1)).
\end{align*}
If $x \in [2(j-1),2j-1]$ for some $j \in [n]$, then
\begin{align*}
\widetilde{\mathcal{N}}_1 (x) &= \sum_{q=1}^{j-1} 3^{-(q-1) d_x} \widetilde{\phi}_K (x-2(q-1)) + 3^{-(j-1) d_x} \widetilde{\phi}_K (x-2(j-1)) \\
&~~~ + \sum_{q=j+1}^{n} 3^{-(q-1) d_x} \widetilde{\phi}_K (x-2(q-1)) \\
&= \sum_{q=1}^{j-1} 3^{-(q-1) d_x} + 3^{-(j-1) d_x} \widetilde{\phi}_K (x-2(j-1)) \\
&= 3^{-(j-1) d_x} \widetilde{\phi}_K (x-2(j-1)) + b_j,
\end{align*}
where we define $b_1 = 0$ and $b_j = \sum_{q=1}^{j-1} 3^{-(q-1) d_x}$ for $j \geq 2$.

Now consider $\bm{x} \in \mathbb{R}^{d_x}$. Fixing $j \in [n]$, if $x_i \in [2(j-1),2j-1]$ for all $i \in [d_x]$, we can construct an FNN $\widetilde{\mathcal{N}}_2 \in \mathcal{FNN}_{d_x,d_x}(4 d_x n,2K)$ such that
\begin{align*}
\widetilde{\mathcal{N}}_2 (\bm{x}) &= \bm{1}_{d_x} \left(1, 3^{-1}, \ldots, 3^{1-d_x}\right) \begin{pmatrix}
\widetilde{\mathcal{N}}_1 (x_1) \\
\widetilde{\mathcal{N}}_1 (x_2) \\
\vdots \\
\widetilde{\mathcal{N}}_1 (x_{d_x})
\end{pmatrix} \\
&= \bm{1}_{d_x} \left(1, 3^{-1}, \ldots, 3^{1-d_x}\right) \begin{pmatrix}
3^{-(j-1) d_x} \widetilde{\phi}_K (x_1 - 2(j-1)) + b_j \\
3^{-(j-1) d_x} \widetilde{\phi}_K (x_2 - 2(j-1)) + b_j \\
\vdots \\
3^{-(j-1) d_x} \widetilde{\phi}_K (x_{d_x} - 2(j-1)) + b_j
\end{pmatrix} \\
&= \left(\sum_{p=1}^{d_x} 3^{1 - p - (j-1) d_x} \widetilde{\phi}_K (x_p - 2(j-1))\right) \bm{1}_{d_x} + \left(b_j \sum_{p=1}^{d_x} 3^{1-p}\right) \bm{1}_{d_x} \\
&= \left(3 \sum_{p=1}^{d_x} a_{p,j} \widetilde{\phi}_K (x_p - 2(j-1))\right) \bm{1}_{d_x} + c_j \bm{1}_{d_x},
\end{align*}
where we define $c_j = b_j \sum_{p=1}^{d_x} 3^{1-p}$. Using that $X_{p,j} \in [0,1]$ implies $X_{p,j} + 2 (j-1) \in [2(j-1),2j-1]$, set $\bm{x} = \bm{X}_{:,j} + 2 (j-1) \bm{1}_{d_x}$ in the above equation to obtain
\begin{align*}
\widetilde{\mathcal{N}}_2 (\bm{X}_{:,j} + 2 (j-1) \bm{1}_{d_x}) &= \left(3 \sum_{p=1}^{d_x} a_{p,j} \widetilde{\phi}_K (X_{p,j})\right) \bm{1}_{d_x} + c_j \bm{1}_{d_x}.
\end{align*}
By Lemma \ref{lemma: 10}, there exist $2K$ feed-forward layers $\mathcal{F}_2^{(FF)}, \ldots, \mathcal{F}_{2K+1}^{(FF)}$, each with width at most $3 \cdot 4 d_x n = 12 d_x n$, such that 
\begin{align*}
\mathcal{F}_{2K+1}^{(FF)} \circ \cdots \circ \mathcal{F}_2^{(FF)} \left(\bm{X}_{:,1}, \ldots, \bm{X}_{:,n} + 2 (n-1) \bm{1}_{d_x}\right) &= \left(\widetilde{\mathcal{N}}_2(\bm{X}_{:,1}), \ldots, \widetilde{\mathcal{N}}_2(\bm{X}_{:,n} + 2 (n-1) \bm{1}_{d_x})\right),
\end{align*}
where we omit placeholder zeros for simplicity. Finally, to add and then remove the bias terms, we use two generalized feed-forward layers. We define
\begin{align*}
\mathcal{F}_1^{(GFF)} (\bm{X}_{:,1}, \ldots, \bm{X}_{:,n}) := (\bm{X}_{:,1}, \ldots, \bm{X}_{:,n} + 2 (n-1) \bm{1}_{d_x})
\end{align*}
and 
\begin{align*}
\mathcal{F}_{2K+2}^{(GFF)} (\bm{Z}_{:,1}, \ldots, \bm{Z}_{:,n}) := (\bm{Z}_{:,1} - c_1 \bm{1}_{d_x}, \ldots, \bm{Z}_{:,n} - c_n \bm{1}_{d_x}).
\end{align*}
It is straightforward to verify that
\begin{align*}
& \mathcal{F}_{2K+2}^{(GFF)} \circ \mathcal{F}_{2K+1}^{(FF)} \circ \cdots \circ \mathcal{F}_2^{(FF)} \circ \mathcal{F}_1^{(GFF)} (\bm{X}) \\
&= \mathcal{F}_{2K+2}^{(GFF)} \circ \mathcal{F}_{2K+1}^{(FF)} \circ \cdots \circ \mathcal{F}_2^{(FF)} \left(\bm{X}_{:,1}, \ldots, \bm{X}_{:,n} + 2 (n-1) \bm{1}_{d_x}\right) \\
&= \mathcal{F}_{2K+2}^{(GFF)} \left(\widetilde{\mathcal{N}}_2(\bm{X}_{:,1}), \ldots, \widetilde{\mathcal{N}}_2(\bm{X}_{:,n} + 2 (n-1) \bm{1}_{d_x})\right) \\
&= \mathcal{F}_{2K+2}^{(GFF)} \left(\left(3 \sum_{p=1}^{d_x} a_{p,1} \widetilde{\phi}_K (X_{p,1})\right) \bm{1}_{d_x} + c_1 \bm{1}_{d_x}, \ldots, \left(3 \sum_{p=1}^{d_x} a_{p,n} \widetilde{\phi}_K (X_{p,n})\right) \bm{1}_{d_x} + c_n \bm{1}_{d_x}\right) \\
&= \left(\left(3 \sum_{p=1}^{d_x} a_{p,1} \widetilde{\phi}_K (X_{p,1})\right) \bm{1}_{d_x}, \ldots, \left(3 \sum_{p=1}^{d_x} a_{p,n} \widetilde{\phi}_K (X_{p,n})\right) \bm{1}_{d_x}\right) \\
&= \bm{Z}_1.
\end{align*}

\textbf{Step 2:} We show that there exist a generalized self-attention layer $\mathcal{F}^{(GSA)}$ and a generalized feed-forward layer $\mathcal{F}_{2K+3}^{(GFF)}$ such that 
\begin{align*}
\mathcal{F}_{2K+3}^{(GFF)} \circ \mathcal{F}^{(GSA)}: \begin{pmatrix}
\bm{Z}_1 \\
\bm{O}
\end{pmatrix} \mapsto \begin{pmatrix}
\bm{Z}_2 \\
\bm{O}
\end{pmatrix},
\end{align*}
where 
\begin{align*}
\bm{Z}_2 = 
\footnotesize \begin{pmatrix}
3 \displaystyle\sum_{p=1}^{d_x} \sum_{q=1}^{n} a_{p,q} \widetilde{\phi}_K (X_{p,q}) & 3 \displaystyle\sum_{p=1}^{d_x} \sum_{q=1}^{n} a_{p,q} \widetilde{\phi}_K (X_{p,q}) + 2 & \cdots & 3 \displaystyle\sum_{p=1}^{d_x} \sum_{q=1}^{n} a_{p,q} \widetilde{\phi}_K (X_{p,q}) + 2 (n-1) \\
\vdots & \vdots & & \vdots \\
3 \displaystyle\sum_{p=1}^{d_x} \sum_{q=1}^{n} a_{p,q} \widetilde{\phi}_K (X_{p,q}) & 3 \displaystyle\sum_{p=1}^{d_x} \sum_{q=1}^{n} a_{p,q} \widetilde{\phi}_K (X_{p,q}) + 2 & \cdots & 3 \displaystyle\sum_{p=1}^{d_x} \sum_{q=1}^{n} a_{p,q} \widetilde{\phi}_K (X_{p,q}) + 2 (n-1)
\end{pmatrix}.
\end{align*}
Note that $\bm{Z}_2$ is obtained by summing the columns of $\bm{Z}_1$ and then adding different bias terms to each column.

We now prove the existence of such layers by first considering a standard self-attention layer. In fact, we only require the softmax function to compute the column average, so it can be replaced by a generalized self-attention layer. We define a self-attention layer by choosing the parameters as follows:
\begin{align*}
H = 1, \quad S = d_x, \quad
\bm{W}^{(O)} = n \begin{pmatrix}
\bm{O}_{d_x} \\
\bm{I}_{d_x}
\end{pmatrix}, \quad
\bm{W}^{(V)} = \left(\bm{I}_{d_x}, \bm{O}_{d_x}\right), \quad
\bm{W}^{(K)} = \bm{O}, \quad
\bm{W}^{(Q)} = \bm{O}.
\end{align*}
Then, by direct calculation based on the definition, we have
\begin{align*}
\mathcal{F}^{(SA)} \begin{pmatrix}
\bm{Z}_1 \\
\bm{O}
\end{pmatrix} = \begin{pmatrix}
\left(3 \sum_{p=1}^{d_x} a_{p,1} \widetilde{\phi}_K (X_{p,1})\right) \bm{1}_{d_x} & \cdots & \left(3 \sum_{p=1}^{d_x} a_{p,n} \widetilde{\phi}_K (X_{p,n})\right) \bm{1}_{d_x} \\
\left(3 \sum_{p=1}^{d_x} \sum_{q=1}^{n} a_{p,q} \widetilde{\phi}_K (X_{p,q})\right) \bm{1}_{d_x} & \cdots & \left(3 \sum_{p=1}^{d_x} \sum_{q=1}^{n} a_{p,q} \widetilde{\phi}_K (X_{p,q})\right) \bm{1}_{d_x}
\end{pmatrix}.
\end{align*}
Next, we define a generalized feed-forward layer with the following parameters:
\begin{align*}
\begin{gathered}
\bm{W}^{(1)} = \begin{pmatrix}
\bm{I}_{d_x} & \bm{O}_{d_x} \\
-\bm{I}_{d_x} & \bm{O}_{d_x} \\
\bm{O}_{d_x} & \bm{I}_{d_x} \\
\bm{O}_{d_x} & -\bm{I}_{d_x} \\
\end{pmatrix}, \quad
\bm{B}^{(1)} = \bm{O}, \\
\bm{W}^{(2)} = \begin{pmatrix}
-\bm{I}_{d_x} & \bm{I}_{d_x} & \bm{I}_{d_x} & -\bm{I}_{d_x} \\
\bm{O}_{d_x} & \bm{O}_{d_x} & -\bm{I}_{d_x} & \bm{I}_{d_x}
\end{pmatrix}, \quad
\bm{B}^{(2)} = \begin{pmatrix}
\bm{0}_{d_x} & 2 \bm{1}_{d_x} & \cdots & 2 (n-1) \bm{1}_{d_x} \\
\bm{0}_{d_x} & \bm{0}_{d_x} & \cdots &\bm{0}_{d_x}
\end{pmatrix}.
\end{gathered}
\end{align*}
It can then be verified that
\begin{align*}
\mathcal{F}_{2K+3}^{(GFF)} \circ \mathcal{F}^{(SA)} \begin{pmatrix}
\bm{Z}_1 \\
\bm{O}
\end{pmatrix} = \begin{pmatrix}
\bm{Z}_2 \\
\bm{O}
\end{pmatrix},
\end{align*}
where we have used the identity $x = \sigma_R[x] - \sigma_R[-x]$.

\textbf{Step 3:} We construct a generalized feed-forward layer $\mathcal{F}_{2K+4}^{(GFF)}$ interpolating the outer function $\bm{G}$ in (\ref{eq: 22}) at the $2^{d_x n K} + 1$ interpolation points
\begin{align*}
& 3 \sum_{p=1}^{d_x} \sum_{q=1}^{n} a_{p,q} \phi_K (X_{p,q}) \in \left\{\sum_{j=1}^{d_x n K} 2 t_j 3^{-j}: (t_1, \ldots, t_{d_x n K}) \in \{0,1\}^{d_x n K}\right\} \bigcup \{1\}.
\end{align*}

Denote these points by $0 =: s_0 < s_1 < \cdots < s_{2^{d_x n K}-1} < s_{2^{d_x n K}} := 1$ and fix $u \in [d_x]$. For any $x \in \mathbb{R}$, we define a scalar function
\begin{align*}
&\widetilde{G}_u (x) := G_{u,1}(s_0) + \sum_{j=1}^{2^{d_x n K}} \frac{G_{u,1}(s_j) - G_{u,1}(s_{j-1})}{s_j - s_{j-1}} \left(\sigma_R[x - s_{j-1}] - \sigma_R[x - s_j]\right) \\
&~~~ + \frac{G_{u,2}(s_0) - G_{u,1}(s_{2^{d_x n K}})}{s_0 + 2 - s_{2^{d_x n K}}} \left(\sigma_R[x - s_{2^{d_x n K}}] - \sigma_R[x - (s_0 + 2)]\right) \\
&~~~ + \sum_{j=1}^{2^{d_x n K}} \frac{G_{u,2}(s_j) - G_{u,2}(s_{j-1})}{s_j - s_{j-1}} \left(\sigma_R[x - (s_{j-1} + 2)] - \sigma_R[x - (s_j + 2)]\right) + \cdots \\
&~~~ + \frac{G_{u,n}(s_0) - G_{u,n-1}(s_{2^{d_x n K}})}{s_0 + 2 - s_{2^{d_x n K}}} \left(\sigma_R[x - (s_{2^{d_x n K}} + 2(n-2))] - \sigma_R[x - (s_0 + 2(n-1))]\right) \\
&~~~ + \sum_{j=1}^{2^{d_x n K}} \frac{G_{u,n}(s_j) - G_{u,n}(s_{j-1})}{s_j - s_{j-1}} \left(\sigma_R[x - (s_{j-1} + 2(n-1))] - \sigma_R[x - (s_j + 2(n-1))]\right) \\
&= G_{u,1}(s_0) \\
&~~~ + \sum_{v=1}^{n} \sum_{j=1}^{2^{d_x n K}} \frac{G_{u,v}(s_j) - G_{u,v}(s_{j-1})}{s_j - s_{j-1}} \left(\sigma_R[x - (s_{j-1} + 2(v-1))] - \sigma_R[x - (s_j + 2(v-1))]\right) \\
&~~~ + \sum_{v=2}^{n} \frac{G_{u,v}(s_0) - G_{u,v-1}(s_{2^{d_x n K}})}{s_0 + 2 - s_{2^{d_x n K}}} \left(\sigma_R[x - (s_{2^{d_x n K}} + 2(v-2))] - \sigma_R[x - (s_0 + 2(v-1))]\right).
\end{align*}
In other words, $\widetilde{G}_u$ is defined as the piecewise linear interpolation of the points 
\begin{align*}
\left\{(s_j + 2(v-1), G_{u,v}(s_j)): j = 0, 1, \ldots, 2^{d_x n K}, \; v = 1, \ldots, n\right\},
\end{align*}
with the function being constant outside the interval $[0,2n-1]$. We observe that
\begin{itemize}[itemsep=0em, labelwidth=1em, leftmargin=!]
\item $\widetilde{G}_u (s_j + 2(v-1)) = G_{u,v}(s_j)$ for every $j \in \{0\} \cup [2^{d_x n K}]$, $u \in [d_x]$, $v \in [n]$,
\item $\|\widetilde{G}_u\|_{L^{\infty}(\mathbb{R})} \leq \max_{v \in [n]} \|G_{u,v}\|_{L^{\infty}(\mathcal{C})}$,
\item $\widetilde{G}_u \in \mathcal{FNN}_{1,1}(n(2^{d_x n K}+1), 1)$.
\end{itemize}
By stacking the functions $\widetilde{G}_u$ for $u \in [d_x]$ vertically, we obtain a feed-forward layer, with width at most $d_x n (2^{d_x n K}+1) + 2 d_x$, such that
\begin{align*}
\mathcal{F}_{2K+4}^{(GFF)}(\bm{Z}) = \begin{pmatrix}
\widetilde{G}_1(Z_{1,1}) & \widetilde{G}_1(Z_{1,2}) & \cdots & \widetilde{G}_1(Z_{1,n}) \\
\widetilde{G}_2(Z_{2,1}) & \widetilde{G}_2(Z_{2,2}) & \cdots & \widetilde{G}_2(Z_{2,n}) \\
\vdots & \vdots & & \vdots \\
\widetilde{G}_{d_x}(Z_{d_x,1}) & \widetilde{G}_{d_x}(Z_{d_x,2}) & \cdots & \widetilde{G}_{d_x}(Z_{d_x,n})
\end{pmatrix}.
\end{align*}

Together with \textbf{Step 1} and \textbf{Step 2}, we define the overall generalized Transformer network as
\begin{align}\label{eq: 24}
\begin{aligned}
\mathcal{N} & := \mathcal{E}_{out} \circ \mathcal{F}_{2K+4}^{(GFF)} \circ \mathcal{F}_{2K+3}^{(GFF)} \circ \mathcal{F}^{(GSA)} \circ \mathcal{F}_{2K+2}^{(GFF)} \circ \cdots \circ \mathcal{F}_1^{(GFF)} \circ \mathcal{E}_{in} \\
&~ \in \mathcal{GT}_{d_x, d_x}(D = 4 d_x n, H = 1, S = d_x, W = d_x n (2^{d_x n K}+1) + 2 d_x, L = 2K+4),
\end{aligned}
\end{align}
where $\mathcal{E}_{in}$ and $\mathcal{E}_{out}$ are appropriately chosen to add and remove zeros to match the hidden dimension $D$. In particular, we have
\begin{align*}
& \mathcal{N}(\bm{X}) \\
&= \mathcal{E}_{out} \circ \mathcal{F}_{2K+4}^{(GFF)} \circ \mathcal{F}_{2K+3}^{(GFF)} \circ \mathcal{F}^{(GSA)} \circ \mathcal{F}_{2K+2}^{(GFF)} \circ \cdots \circ \mathcal{F}_1^{(GFF)} \circ \mathcal{E}_{in} (\bm{X}) \\
&= \mathcal{E}_{out} \circ \mathcal{F}_{2K+4}^{(GFF)} \circ \mathcal{F}_{2K+3}^{(GFF)} \circ \mathcal{F}^{(GSA)} \begin{pmatrix}
\bm{Z}_1 \\
\bm{O}
\end{pmatrix} \\
&= \mathcal{E}_{out} \circ \mathcal{F}_{2K+4}^{(GFF)} \begin{pmatrix}
\bm{Z}_2 \\
\bm{O}
\end{pmatrix} \\
&= \begin{aligned}
\footnotesize \begin{pmatrix} 
\widetilde{G}_1\left(3 \displaystyle\sum_{p=1}^{d_x} \sum_{q=1}^{n} a_{p,q} \widetilde{\phi}_K (X_{p,q})\right) & \widetilde{G}_1\left(3 \displaystyle\sum_{p=1}^{d_x} \sum_{q=1}^{n} a_{p,q} \widetilde{\phi}_K (X_{p,q}) + 2\right) & \cdots & \widetilde{G}_1\left(3 \displaystyle\sum_{p=1}^{d_x} \sum_{q=1}^{n} a_{p,q} \widetilde{\phi}_K (X_{p,q}) + 2 (n-1)\right) \\
\widetilde{G}_2\left(3 \displaystyle\sum_{p=1}^{d_x} \sum_{q=1}^{n} a_{p,q} \widetilde{\phi}_K (X_{p,q})\right) & \widetilde{G}_2\left(3 \displaystyle\sum_{p=1}^{d_x} \sum_{q=1}^{n} a_{p,q} \widetilde{\phi}_K (X_{p,q}) + 2\right) & \cdots & \widetilde{G}_2\left(3 \displaystyle\sum_{p=1}^{d_x} \sum_{q=1}^{n} a_{p,q} \widetilde{\phi}_K (X_{p,q}) + 2 (n-1)\right) \\
\vdots & \vdots & & \vdots \\
\widetilde{G}_{d_x}\left(3 \displaystyle\sum_{p=1}^{d_x} \sum_{q=1}^{n} a_{p,q} \widetilde{\phi}_K (X_{p,q})\right) & \widetilde{G}_{d_x}\left(3 \displaystyle\sum_{p=1}^{d_x} \sum_{q=1}^{n} a_{p,q} \widetilde{\phi}_K (X_{p,q}) + 2\right) & \cdots & \widetilde{G}_{d_x}\left(3 \displaystyle\sum_{p=1}^{d_x} \sum_{q=1}^{n} a_{p,q} \widetilde{\phi}_K (X_{p,q}) + 2 (n-1)\right)
\end{pmatrix}.
\end{aligned}
\end{align*}

\textbf{Step 4:} We now conduct an error analysis. We have
\begin{align*}
& \|\mathcal{N} - \bm{F}\|_{L^p([0,1]^{d_x \times n})}^p \\
&= \int_{[0,1]^{d_x \times n}} \|\mathcal{N}(\bm{X}) - \bm{F}(\bm{X})\|_F^p \dd \bm{X} \\
&\leq \int_{[0,1]^{d_x \times n}} (d_x n)^{\max\{0,\frac{p}{2}-1\}} \sum_{u=1}^{d_x} \sum_{v=1}^{n} |\mathcal{N}_{u,v}(\bm{X}) - F_{u,v}(\bm{X})|^p \dd \bm{X} \\
&= (d_x n)^{\max\{0,\frac{p}{2}-1\}} \sum_{u=1}^{d_x} \sum_{v=1}^{n} \left(\int_{\bm{X}: \forall X_{i,j} \in \Omega_K} + \int_{\bm{X}: \exists X_{i,j} \notin \Omega_K}\right) |\mathcal{N}_{u,v}(\bm{X}) - F_{u,v}(\bm{X})|^p \dd \bm{X} \\
&=: (d_x n)^{\max\{0,\frac{p}{2}-1\}} \sum_{u=1}^{d_x} \sum_{v=1}^{n} (\Rmnum{1} + \Rmnum{2}).
\end{align*}

To estimate $\Rmnum{1}$, using that $\widetilde{\phi}_K (X_{p,q}) = \phi_K (X_{p,q})$ when $X_{p,q} \in \Omega_K$, $\widetilde{G}_{u}$ interpolates $G_{u,v}$ by construction, and $G_{u,v} \in \mathcal{H}^{\frac{\gamma \log 2}{d_x n \log 3}}(\mathcal{C}, 2 \sqrt{d_x n} K_\mathcal{H})$, we have
\begin{align*}
& |\mathcal{N}_{u,v}(\bm{X}) - F_{u,v}(\bm{X})| \\
&= \left|\widetilde{G}_{u}\left(3 \sum_{p=1}^{d_x} \sum_{q=1}^{n} a_{p,q} \widetilde{\phi}_K (X_{p,q}) + 2 (v-1)\right) - G_{u,v} \left(3 \sum_{p=1}^{d_x} \sum_{q=1}^{n} a_{p,q} \phi (X_{p,q})\right)\right| \\
&= \left|\widetilde{G}_{u}\left(3 \sum_{p=1}^{d_x} \sum_{q=1}^{n} a_{p,q} \phi_K (X_{p,q}) + 2 (v-1)\right) - G_{u,v} \left(3 \sum_{p=1}^{d_x} \sum_{q=1}^{n} a_{p,q} \phi (X_{p,q})\right)\right| \\
&= \left|G_{u,v}\left(3 \sum_{p=1}^{d_x} \sum_{q=1}^{n} a_{p,q} \phi_K (X_{p,q})\right) - G_{u,v} \left(3 \sum_{p=1}^{d_x} \sum_{q=1}^{n} a_{p,q} \phi (X_{p,q})\right)\right| \\
&\leq 2 (d_x n)^{\frac{1}{2}} K_\mathcal{H} \left|3 \sum_{p=1}^{d_x} \sum_{q=1}^{n} a_{p,q} \left(\phi_K (X_{p,q}) - \phi (X_{p,q})\right)\right|^{\frac{\gamma \log 2}{d_x n \log 3}} \\
&\leq 2 (d_x n)^{\frac{1}{2}} K_\mathcal{H} \left|2 \sum_{q = d_x n K + 1}^{\infty} 3^{-q}\right|^{\frac{\gamma \log 2}{d_x n \log 3}} \\
&\leq 2 (d_x n)^{\frac{1}{2}} K_\mathcal{H} 3^{-\frac{K \gamma \log 2}{\log 3}} \\
&= 2 (d_x n)^{\frac{1}{2}} K_\mathcal{H} 2^{-\gamma K},
\end{align*}
where the second inequality follows from the fact that, as indicated in (\ref{eq: 23}), $3 \sum_{p=1}^{d_x} \sum_{q=1}^{n} a_{p,q} \phi (X_{p,q})$ and $3 \sum_{p=1}^{d_x} \sum_{q=1}^{n} a_{p,q} \phi_K (X_{p,q})$ are both in the Cantor set $\mathcal{C}$ and have the same first $d_x n K$ ternary digits. Thus,
\begin{align*}
\Rmnum{1} &= \int_{\bm{X}: \forall X_{i,j} \in \Omega_K} |\mathcal{N}_{u,v}(\bm{X}) - F_{u,v}(\bm{X})|^p \dd \bm{X} \\
&\leq 2^p (d_x n)^{\frac{p}{2}} K_\mathcal{H}^p 2^{-p \gamma K}.
\end{align*}
To estimate $\Rmnum{2}$, noting that both $\mathcal{N}_{u,v}$ and $F_{u,v}$ are bounded, and that $\Omega_K$ has Lebesgue measure at least $1 - 2^{-p \gamma K}$, we obtain
\begin{align*}
\Rmnum{2} &= \int_{\bm{X}: \exists X_{i,j} \notin \Omega_K} |\mathcal{N}_{u,v}(\bm{X}) - F_{u,v}(\bm{X})|^p \dd \bm{X} \\
&\leq \int_{\bm{X}: \exists X_{i,j} \notin \Omega_K} \left(\|\mathcal{N}_{u,v}\|_{L^{\infty}([0,1]^{d_x \times n})} + \|F_{u,v}\|_{L^{\infty}([0,1]^{d_x \times n})}\right)^p \dd \bm{X} \\
&\leq \left(\|\mathcal{N}_{u,v}\|_{L^{\infty}([0,1]^{d_x \times n})} + \|F_{u,v}\|_{L^{\infty}([0,1]^{d_x \times n})}\right)^p (1 - (1 - 2^{-p \gamma K})^{d_x n}) \\
&\leq 2^{2 p} (d_x n)^{\frac{p}{2}+1} K_\mathcal{H}^p 2^{-p \gamma K},
\end{align*}
where we apply Bernoulli's inequality in the last inequality.

We combine the bounds for $\Rmnum{1}$ and $\Rmnum{2}$ to obtain
\begin{align*}
\|\mathcal{N} - \bm{F}\|_{L^p([0,1]^{d_x \times n})}^p &\leq (d_x n)^{\max\{0,\frac{p}{2}-1\}} \sum_{u=1}^{d_x} \sum_{v=1}^{n} \left(2^p (d_x n)^{\frac{p}{2}} K_\mathcal{H}^p 2^{-p \gamma K} + 2^{2 p} (d_x n)^{\frac{p}{2}+1} K_\mathcal{H}^p 2^{-p \gamma K}\right) \\
&\leq (d_x n)^{3 p} 2^{2 p} K_\mathcal{H}^p 2^{-p \gamma K},
\end{align*}
where we have used $p \geq 1$ and $\max\{a, b\} \leq a + b$ for all $a, b \geq 0$, which implies
\begin{align*}
\|\mathcal{N} - \bm{F}\|_{L^p([0,1]^{d_x \times n})} \leq 4 (d_x n)^3 K_\mathcal{H} 2^{-\gamma K}.
\end{align*}

Choose $K \geq \frac{1}{\gamma} \log_2 \frac{1}{\varepsilon}$ so that $2^{d_x n K} = \lceil\varepsilon^{-\frac{d_x n}{\gamma}}\rceil$. Then we have
\begin{align*}
\|\mathcal{N} - \bm{F}\|_{L^p([0,1]^{d_x \times n})} \leq 4 (d_x n)^3 K_\mathcal{H} \varepsilon,
\end{align*}
and by (\ref{eq: 24}),
\begin{align*}
\mathcal{N} \in \mathcal{GT}_{d_x, d_x}(D = 4 d_x n, H = 1, S = d_x, W \leq 3 d_x n \lceil\varepsilon^{-\frac{d_x n}{\gamma}}\rceil, L \leq 6 \lceil\tfrac{1}{\gamma} \log_2 \tfrac{1}{\varepsilon}\rceil).
\end{align*}
This completes the proof.
\end{proof}

\bibliographystyle{plain}
\bibliography{reference}

\end{document}